\documentclass[11pt]{article}

\usepackage{fullpage}

\usepackage{amsmath,amsthm,amsfonts,amssymb}
\usepackage{amsmath}
\usepackage{hyperref}
\usepackage{color}
\usepackage{mathrsfs}
\usepackage{bm}
\usepackage{multirow}
\usepackage{booktabs}
\usepackage{makecell}
\usepackage{graphicx}
\usepackage{caption}
\usepackage{subcaption}
\usepackage{thmtools}
\usepackage{thm-restate}
\usepackage{setspace}
\usepackage{makecell}
\definecolor{light-gray}{gray}{0.85}
\usepackage{colortbl}
\usepackage{xcolor}
\usepackage{hhline}

\usepackage{float}

\hypersetup{
    colorlinks,
    linkcolor={blue!50!black},
    citecolor={blue!50!black},
}

\usepackage[capitalize,noabbrev]{cleveref}

\usepackage{amsfonts}

\usepackage{times}

\theoremstyle{plain}
\newtheorem{theorem}{Theorem}[section]
\newtheorem{proposition}[theorem]{Proposition}
\newtheorem{lemma}[theorem]{Lemma}
\newtheorem{claim}[theorem]{Claim}
\newtheorem{corollary}[theorem]{Corollary}
\theoremstyle{definition}
\newtheorem{definition}[theorem]{Definition}
\newtheorem{assumption}[theorem]{Assumption}
\theoremstyle{remark}
\newtheorem{remark}[theorem]{Remark}

\usepackage{mathrsfs}
\usepackage{mathnotations}
\usepackage{algorithm}
\usepackage{algorithmic}
\usepackage{subcaption, wrapfig}

\begin{document}

\title{Shuffle Private Linear Contextual Bandits}

\author {
    % Authors
    Sayak Ray Chowdhury\thanks{Equal contributions} \footnote{Boston University, Massachusetts, USA. Email: \texttt{sayak@bu.edu}  }\quad
    Xingyu Zhou\footnotemark[1] \footnote{Wayne State University, Detroit, USA.  Email: \texttt{xingyu.zhou@wayne.edu}}
}

\date{}

\maketitle

\begin{abstract}
Differential privacy (DP) has been recently introduced to linear contextual bandits to formally address the privacy concerns in its associated personalized services  to participating users (e.g., recommendations).
Prior work largely focus on two trust models of DP -- the central model, where a central server is responsible for protecting users’ sensitive data, and the (stronger) local model, where information needs to be protected directly on users' side. However, there remains a fundamental gap in the utility achieved by learning algorithms under these two privacy models, 
e.g., $\widetilde{O}(\sqrt{T})$ regret in the central model as compared to $\widetilde{O}(T^{3/4})$ regret in the local model, if all users are \emph{unique} within a learning horizon $T$. In this work, we aim to achieve a stronger model of trust than the central model, while suffering a smaller regret than the local model by considering recently popular \emph{shuffle} model of privacy. We propose a general algorithmic framework for linear contextual bandits under the shuffle trust model, where there exists a trusted shuffler  -- in between users and the central server-- that randomly permutes a batch of users data before sending those to the server. 
We then instantiate this framework with two specific shuffle protocols -- one relying on privacy amplification of local mechanisms, and another incorporating a protocol for summing vectors and matrices of bounded norms. We prove that both these instantiations lead to regret guarantees that significantly improve on that of the local model, and can potentially be of the order $\widetilde{O}(T^{3/5})$ if all users are unique. We also verify this regret behavior with simulations on synthetic data. Finally, under the practical scenario of non-unique users, we show that the regret of our shuffle private algorithm scale as $\widetilde{O}(T^{2/3})$, which \emph{matches} that the central model could achieve in this case. 
\end{abstract}
% \xingyu{Can we use small fonts in abstract?}

\section{Introduction}

In the linear contextual bandit problem \cite{Auer03confidence,Chu2011}, a learning agent observes the context information $c_t$ of an user at every round $t$. The goal is to recommend an action $a_t$ to the user so that the resulting reward $y_t$ is maximized. The mean reward is given by a linear function of an \emph{unknown} parameter vector $\theta^* \in \Real^d$, $d \in \Nat$, i.e., 
\begin{align*}
    \expect{y_t\mid c_t,a_t} = \inner{\theta^*}{\phi(c_t,a_t)}~,
\end{align*}
where $\phi: \cC \times \cX \to \Real^d$ maps a context-action pair to a $d$-dimensional feature vector, and $\inner{\cdot}{\cdot}$ denotes the standard Euclidean inner product. The context and action sets $\cC$ and $\cX$ are arbitrary, and can also possibly be varying with time.
An agent's performance over $T$ rounds is typically measured through the cumulative pseudo-regret 
\begin{align*}
    \reg(T) = \sum\nolimits_{t=1}^T\left[ \max_{a \in \cX} \inner{\theta^*}{\phi(c_t,a)} -\inner{\theta^*}{\phi(c_t, a_t)}\right],
\end{align*}
which is the total loss suffered due to not recommending the actions generating highest possible rewards corresponding to observed contexts. This framework has found applications in many real-life settings such as internet advertisement selection \cite{abe2003reinforcement}, article recommendation in web portals \cite{li2010contextual}, mobile health \cite{tewari2017ads}, to name a few.
The general applicability
of this framework has motivated a line of work \cite{shariff2018differentially,zheng2020locally} studying linear contextual bandit problems under
the additional constraint of \emph{differential privacy} \cite{dwork2008differential}, which guarantees that the users' contexts and generated rewards will not be inferred by an adversary during this learning process.
% the recommended actions
% reveal little information about .

To illustrate the privacy concern in the contextual bandit problem, let us consider a mobile medical application in which an mobile app recommends a tailored treatment plan (i.e., action) to each patient (i.e., user) based on her personal information such as age, weight, height, medical history etc. (i.e., context). Meanwhile, this mobile app's recommendation algorithm also needs to be updated once in a while in a cloud server after collecting data from a batch of patients, including treatment outcomes (i.e., rewards) and contexts, which are often considered to be private and sensitive information. Hence, each patient would like to obtain a personalized and effective treatment plan while guaranteeing their sensitive information remains protected against a potential adversarial attack in this interactive process. 
% To this end, \emph{differential privacy} has been recently introduced to 
Protection of privacy is typically achieved by injecting sufficient noise in users' data \cite{arora2014privacy,xin2014controlling}, which results in a loss in utility (i.e., an increase in regret) of the recommended action. Hence, the key question is how to balance utility and privacy carefully.

% The deployed algorithm (e.g., local algorithm in the mobile app) will be constantly updated by the learning agent (e.g., cloud server) relying on history interactive data with users such as feature vector $\phi(c,a)$ and rewards. In this process, an adversary could possibly attack the learning protocol to infer a particular user's sensitive personal information. To formally alleviate the above privacy concerns, Differential privacy has become a standard criterion in designing private contextual bandit algorithms. Depending on the attack point (or equivalently the adversary's view), we have different notions of differential privacy (DP), which in turn assume different trust models, i.e., who the user can trust trusts with her sensitive data.

This has motivated studies of linear contextual bandits under different trust models of differential privacy (i.e., who the user will trust with her sensitive data). On one end of the spectrum lies the \emph{central} model, which guarantees privacy to users who trust the learning agent to store their raw data in the server and use those to update its strategy of recommending actions. Under this trust model,~\cite{shariff2018differentially} has shown that the cumulative regret is  
$\widetilde{O}(\frac{\sqrt{T}(\log(1/\delta))^{1/4}}{\sqrt{\epsilon}})$, where $\epsilon$ and $\delta$ are privacy parameters with smaller values denoting higher level of protection. Perhaps unsurprisingly, this regret bound -- due to the high degree of trust -- matches the optimal $\Theta(\sqrt{T})$ scaling for non-private linear contextual bandits \cite{Chu2011}.
However, this relatively high trust model is not always   feasible since the users may not trust the agent at all.
This is captured by the \emph{local} model, where any data sent by the users must already be private, and the agent can only store those randomized data in the server. This is  a strictly stronger notion of privacy, and hence, often comes at a price. Under this trust model, \cite{zheng2020locally} has shown that the cumulative regret is $\widetilde{O}(\frac{T^{3/4}(\log(1/\delta))^{1/4}}{\sqrt{\epsilon}})$, which, as expected, is much worse than that in the central model.
 This naturally leads to the following question: 
\begin{center}
    \vspace{-2mm}
   \emph{Can a finer trade-off between privacy and regret in linear contextual bandits be achieved?} 
  \vspace{-2mm}
\end{center}
Furthermore, in both \cite{shariff2018differentially} and \cite{zheng2020locally}, the learning agents update their strategy at every round. This not only puts excessive computational burden on the server (due to $T$ updates each taking at least $O(d^2)$ time and memory) but also could be be practically infeasible at times. For example, consider the above mobile health application. The cloud server is often infeasible to update the algorithm deployed in mobile app after interactions with each single user. Rather, a more practical strategy is to update the algorithm after collecting a batch of users' data (e.g., a one-month period of data). 

% \footnote{\url{https://fuchsia.googlesource.com/cobalt/}}

Motivated by these, we consider the linear contextual bandit problem under an intermediate trust model of differential privacy, known as the \emph{shuffle} model \cite{cheu2019distributed,erlingsson2019amplification} in the hope to attain a finer regret-privacy trade-off, while only using batch updates. In this new trust model, there exists a shuffler between users and the central server which permutes a \emph{batch} of users' randomized data before they are viewed by the server so that it can't distinguish between two users' data. Shuffling thus adds an another
layer of protection by decoupling data from the users that sent them. Here, as in the local model, the users don't trust the server. However, it is assumed that they have a certain degree of trust in the shuffler
since it can be efficiently implemented using cryptographic primitives (e.g., mixnets) due to its simple operation \cite{bittau2017prochlo,apple}. The shuffle model provides the possibility to achieve a stronger privacy guarantee than the central model while suffering a smaller utility loss than the local model. The key intuition behind this is that the additional randomness of the shuffler creates a \emph{privacy blanket}~\cite{balle2019privacy} so that each user now needs much less random noise to hide her information in the crowd.
Indeed, the shuffle model achieves a better trade-off between utility and privacy as compared to central and local model in several learning problems such as empirical risk minimization \cite{girgis2021shuffled}, stochastic convex optimization \cite{Lowy2021,cheu2021shuffle}, and standard multi-arm bandits \cite{tenenbaum2021differentially}.
However, little is known about (linear) contextual bandits in the shuffle model due to its intrinsic challenges. That is, in addition to rewards, the contexts are also sensitive information that need to be protected, which not only results in the aforementioned large gap in regret between local and central model\footnote{In contrast, for MAB, the problem-independent upper bounds in the local and central model are both $\widetilde{O}(\sqrt{T})$~\cite{ren2020multi}.}, but also leads to new challenges in the shuffle model. Against this backdrop, we make the following contributions:

% \paragraph{Outline and contributions}

% However, note that there is a fundamental difference between MAB and linear contextual bandits in the private case. That is, in addition to rewards, the contexts are also sensitive information that need to be protected, which not only results in the aforementioned large gap between local and central model\footnote{In contrast, for MAB, the problem-independent upper bounds in the local and central model are both $\tilde{O}(\sqrt{T})$~\cite{ren2020multi}}, but leads to new challenges in the shuffle model.

\begin{itemize}
\vspace{-1.5mm}
    % \item We design a general algorithmic framework (Algorithm~\ref{alg:BOFUL}) for private linear contextual bandits in the shuffle model. It decomposes the learning process into three black-box components: a local randomizer at each user, an analyzer at the central server and a shuffler in-between. 
    % \item We instantiate the framework with two specific shuffle protocols: The first one directly builds on privacy amplification of existing local mechanisms, hence minimum modifications. The other utilizes a a particularly efficient and accurate mechanism for summing vectors with bounded $\ell_2$ norms, which enjoys several benefits including communication of bits rather than real numbers and privacy guarantee for a wide range of $\epsilon$.
    \vspace{-1.5mm}
    \item We design a general algorithmic framework (Algorithm~\ref{alg:BOFUL}) for private linear contextual bandits in the shuffle model. It decomposes the learning process into three black-box components: a local randomizer at each user, an analyzer at the central server and a shuffler in-between. We instantiate the framework with two specific shuffle protocols. The first one directly builds on privacy amplification of existing local mechanisms. The other one utilizes an efficient mechanism for summing vectors with bounded $\ell_2$ norms. 
    \vspace{-1.5mm}
    \item We show that both shuffle 
    protocols provide stronger privacy protection compared to the central model. Furthermore, when all users are \emph{unique}, we prove a regret bound of $\widetilde{O}\left( T^{3/5}\right)$ for both the protocols, which improves over the $\widetilde{O}\left( T^{3/4}\right)$ regret of local model. Hence, we achieve a finer trade-off between regret and privacy. We further perform simulations on synthetic data that corroborate our theoretical results.
    % \xingyu{I was wondering if we include $\epsilon$ here since reviewers may say the facebook paper has better dependence than ours if they did not look at that remark carefully (many reviewers may not be careful). We should not hurt ourselves.}
    % \item Our theoretical results show that both instantiations can provide a stronger privacy guarantee than the central model and  achieve a much better regret $\tilde{O}(T^{3/5})$ compared to the local model. Simultaneously, both instantiations also offer a certain degree of local privacy. We also verify this finer regret-privacy trade-off behavior with simulations on synthetic data
    \vspace{-1.5mm}
    \item
    As a practical application of our general framework, we show that under the setting of non-unique or \emph{returning} users, the regret of both our shuffle protocols \emph{matches} the one that the central model would achieve in the same setting. This, along with the fact that both shuffle protocols also offer a certain degree of local privacy, further elaborate usefulness of shuffle model in private linear contextual bandits.
    % \item  Along the way, we provide a 
    % generic template of regret bound for private linear contextual bandits including the shuffle model, which, we believe, could be of interest to the community.
    % unified regret analysis for private (batched) linear contextual bandits in all three trust models and a highlight of the subtlety regarding the adaptive model update in the private case. 
\end{itemize}

% \sayak{Please try to use conference versions in citations if possible.}\xingyu{sure, will update}

\textbf{Related work.} 
Due to the utility gap present between central and local models, a significant body of recent work have focused on the shuffle model~\cite{balle2019privacy,feldman2020hiding,ghazi2019power,balle2019differentially}. A nice overview of recent work in the shuffle model is presented in~\cite{cheu2021differential}.
Regret performance of multi-armed bandit algorithms under central and local trust models have been considered in ~\cite{mishra2015nearly,sajed2019optimal, ren2020multi,chen2020locally,zhou2020local,dubey2021no,tossou2017achieving}, whereas online learning algorithms under full information have appeared in ~\cite{guha2013nearly,agarwal2017price}. Recently, the two models have also been adopted to design differentially private control and reinforcement learning algorithms~\cite{vietri2020private,garcelon2020local,chowdhury2021adaptive,sayakPO}. \cite{han2021generalized} consider linear bandits with stochastic contexts, and show that $\widetilde{O}(\sqrt{T}/\epsilon)$ regret can be achieved even in the local model.
In contrast, in this work, we allow the contexts to be arbitrary and can even be adversarially generated, which pose additional challenges.

Batched linear bandits are studied in~\cite{han2020sequential,ren2020batched}, where the authors show that only $O(\sqrt{T})$ model update is sufficient to achieve corresponding minimax optimal regrets. In the shuffle private model, batched learning not only reduces the model update frequency, but more importantly plays a key role in amplifying privacy via shuffling a batch of users' data. Interestingly, as a by-product, our established generic regret bound also improves over the non-private one in~\cite{ren2020batched} in the sense that no restriction is required for the regularizer.

\textbf{Concurrent and independent work.} While preparing this submission, we have noticed that ~\cite{garcelon2021privacy} also study linear contextual bandits in the shuffle model. The authors claim that a single fixed batch schedule is not sufficient to obtain a better regret-privacy trade-off in shuffle model. They propose to use separate asynchronous schedules --  a fixed batch scheme for the shuffler and an adaptive model update scheme for the server. 
% based on standard determinant condition~\cite{abbasi2011improved}. 
In contrast, thanks to a tighter analysis, we show that \emph{a single fixed batch schedule} is indeed sufficient to attain the same regret-privacy trade-off in shuffle model. Moreover, we believe, there exists a fundamental gap in their analysis for the adaptive model update, which might make their results ungrounded. 
% In fact, we believe that the use of adaptive batch size at the server will at least enlarge the regret bound by a constant factor, which is reflected in the standard sequential case. 
We provide a detailed discussion on this in Section~\ref{sec:conclude}, which highlights the key difference in dealing with adaptive update in the non-private and the private settings. 
% In fact, this subtlety not only exists in the shuffle model, but universal in all three trust models when handling adaptive model update with noisy design matrix via determinant condition. 
% In fact, as already reflected in the sequential update case,  the use of adaptive rate of update at the server will at least enlarge the regret bound by a constant factor due to its infrequent model update.
Finally, in addition to the above differences in theoretical results,
our established generic framework enables to design flexible shuffle private protocols for linear contextual bandits that are able to handle a wide range of practically interested privacy budget $\epsilon$ rather than a restricted small value $\epsilon \ll 1$  in the concurrent work~\cite{garcelon2021privacy}.

\section{Privacy in the Shuffle Model}
In this section, we introduce the shuffle model, and its corresponding privacy notion called the \emph{shuffle differential privacy}. Before that, we recall definitions of differential privacy under central and local models \cite{dwork2014DPbook}.
% \xingyu{I was wondering if  we can try to use the general definition of DP as in the appendix A in \url{https://arxiv.org/pdf/1812.06210.pdf}. This in some sense allows to define various DPs in our case by giving the adjacent relation, especially useful for the section on returning users.} \sayak{We can give. We need to make sure the definition of LDP and SDP do not change as written now (or can be done with minimal change).}\xingyu{Or I can give the general one in the appendix when I want to talk about user-level DP.}\sayak{Yes, that is also a good idea. Because only in one section, we talk about returning users. Before that we can mention this alternative definition in words.}\xingyu{yes, in the appendix, I can give as much detail as possible.}
\subsection{Central and Local Differential Privacy}
 Throughout, we let $\cD$ denote the data universe, and $n \in \Nat$ the number of (unique) users. Let $D_i \in \cD, i =1,2,\ldots, n$, denote the data point of user $i$, and $D_{-i} \in \cD^{n-1}$ denote collection of data points of all but the $i$-th user. Let $\epsilon > 0$ and $\delta \in (0,1]$ be given privacy parameters.

\begin{definition}[Differential Privacy (DP)]
\label{def:DP}
A mechanism $\cM$ satisfies $(\epsilon,\delta)$-DP if for each user $i \in [n]$, each data set $D, D' \in \cD^n$, and each event $\cE$ in the range of $\cM$, 
\begin{align*}
    \prob{\cM(D_i,D_{-i}) \in \cE} \leq \exp(\epsilon) \prob{\cM(D'_i,D_{-i}) \in \cE} + \delta.
\end{align*}
\end{definition}
\begin{definition}[Local Differential Privacy (LDP)]
\label{def:LDP}
A mechanism $\cM$ satisfies $(\epsilon,\delta)$-LDP if for each user $i \in [n]$, each data point $D_i,D'_i \in \cD$ and each event $\cE$ in the range of $\cM$,
 \begin{align*}
    \prob{\cM(D_i) \in \cE} \leq \exp(\epsilon) \prob{\cM(D'_i) \in \cE} + \delta.
\end{align*}
\end{definition}
Roughly speaking, a central DP (or, simply, DP) mechanism ensures that the outputs of the mechanism on two neighbouring data sets (i.e., those differ only on one user) are approximately indistinguishable. In contrast, local DP ensures that the output of the local mechanism for each user is indistinguishable.

% will formally introduce the shuffle model for linear contextual bandits and define a new notion of privacy called \emph{shuffle differential privacy} (SDP) for it. To this end, we present a general algorithmic framework for linear contextual bandits by leveraging batched learning and the shuffle primitive. 
% As we will see, this framework not only allows us to design and analyze various learning protocols in the shuffle model, but more importantly offers the first unified view of linear contextual bandits under three trust models (i.e., JDP, LDP, SDP). 

\subsection{Shuffle Differential Privacy}
% Our general framework builds on batched learning and the shuffle primitive, i.e., a shuffler $\kS$ that randomly permutes a collection (batch) of users' messages to provide anonymity to each of the messages. 
% We briefly introduce the well-studied one-round shuffle model over a data set of $n$ users, which will serve as a building block for our case. 
% In this setting, 
% To better introduce it, we first recall the standard one-round shuffle model . 
A (standard) shuffle protocol $\cP = (\cR, \cS, \cA)$ consists of three parts: (i) a (local) randomizer $\cR$, (ii) a shuffler $\cS$ and (iii) an analyzer $\cA$. For $n$ users, the overall protocol works as follows. Each user $i$ first applies the randomizer on its raw data $D_i$ and then sends the resulting messages $\cR(D_i)$ to the shuffler. The shuffler $\cS$ permutes messages from all the users uniformly at random and then reports the permuted messages $\cS(\cR(D_1), \ldots, \cR(D_n))$ to the analyzer. Finally, the analyzer $\cA$ computes the output using received messages. In this protocol, the users trust the shuffler but not the analyzer. Hence, the privacy objective is to ensure that the outputs of the shuffler on two neighbouring datasets are indistinguishable in the analyzer's view. To this end, define the mechanism $(\cS \circ \cR^n)(D)\!:=\! \cS(\cR(D_1), \ldots, \cR(D_n))$, where $D \in \cD^n$.
\begin{definition}[Shuffle differential privacy (SDP)]
A protocol $\cP = (\cR, \cS, \cA)$ for $n$ users satisfies $(\epsilon,\delta)$-SDP if the mechanism $\cS \circ \cR^n$ satisfies $(\epsilon,\delta)$-DP.
\end{definition}

% This standard shuffle model has been utilized to achieve a better utility-privacy trade-off in many applications, e.g., a utility that is close or equal to the central model while enjoying similar privacy protection as in the local model.  The key intuition behind this is that the additional randomness of the shuffler creates a \emph{privacy blanket} so that each user now only needs much less random noise to hide her information in the crowd, i.e., privacy amplification by shuffling. 

To achieve benefits of the shuffle model in intrinsically adaptive algorithms (e.g., gradient descent, multi-armed bandits etc.), one needs to divide the users into multiple batches, and run a potentially different shuffle protocol
on each batch \cite{cheu2021shuffle,tenenbaum2021differentially}. This is quite natural since the shuffler needs enough users' data to infuse sufficient randomness so as to amplify the privacy. Moreover, each protocol might depend on the output of the preceding protocols to foster adaptivity. Formally, a general $M$-batch, $M \in \Nat$, shuffle protocol $\cP$ for $n$ users works as follows. In each batch $m$, we simply run a standard single-batch shuffle protocol for a subset of $n_m$ users (such that $n=\sum_{m}n_m$) with randomizer $\cR_m$, shuffler $\cS$ and analyzer $\cA$. To ensure adaptivity, the randomizer $\cR_{m}$ and number of users $n_{m}$ for the $m$-th batch could be chosen depending on outputs of the shuffler from all the previous batches, given by $\left\lbrace\cS\left(\cR_{m'}(D_1),\ldots,\cR_{m'}(D_{n_{m'}})\right)\right\rbrace_{m' <m}$. The objective of privacy is same as in the single-batch protocol -- the analyzer's view must satisfy DP.
However, instead of a single-batch output, one need to protect outputs of all the $M$ batches. To this end, define the (composite) mechanism $\cM_{\cP}=(\cS \circ R_1^{n_1},\ldots,\cS \circ \cR_m^{n_m})$, where each individual mechanism $\cS \circ \cR_m^{n_m}$ operates on $n_m$ users' data, i.e., on datasets from $\cD^{n_m}$.
% \xingyu{This is nice. In some sense, $M$ is kind of necessary in such definition. }

\begin{definition}[$M$-batch SDP]
An $M$-batch shuffle protocol $\cP$ is $(\epsilon,\delta)$-SDP if the mechanism $\cM_{\cP}$ is $(\epsilon,\delta)$-DP.
\end{definition}

\begin{algorithm}[tb]
   \caption{Shuffle Private LinUCB}
   \label{alg:BOFUL}
\begin{algorithmic}[1]
   \STATE {\bfseries Parameters:} Batch size $B \in \Nat$, regularization $\lambda \!>\! 0$, confidence radii $\lbrace\beta_m\rbrace_{m \geq 0}$, feature map $\phi:\cC \!\times\! \cX \!\to\! \Real^d$
   \STATE {\bfseries Initialize:} Batch counter $m\!=\!1$, end-time $t_0\!=\!0$, batch statistics ${V}_{0} \!=\! \lambda I_d$, ${u}_{0} \!=\! 0$, parameter estimate $\hat\theta_0\!=\!0$
 \FOR{local user $t\!=\!1, 2,\dots$}
%   \STATE The local algorithm is updated with $\hat{\theta}_m$ and $V_m$ 
   \STATE Observe user's context information $c_t \in \cC$
   \STATE Choose action
  $a_t \in \argmax_{a \in \cX}  \inner{\phi(c_t,a)}{\hat{\theta}_{m-1}} + \beta_{m-1} \norm{\phi(c_t, a)}_{{V}_{m-1}^{-1}} $
   \STATE Observe reward $y_t$
   \STATE \textcolor{gray}{\# For the local randomizer:}
   \STATE Send randomized messages $M_{t,1}=R_1(\phi(c_t,a_t)y_t)$ and $M_{t,2}=R_2(\phi(c_t,a_t)\phi(c_t,a_t)^\top)$ to the shuffler
%   \STATE $M_{t,1}=R_1(\phi(c_t,a_t)y_t)$ 
%   \STATE $M_{t,2}=R_2(\phi(c_t,a_t)\phi(c_t,a_t)^\top)$ 
   \IF{$t = mB$} 
        \STATE \textcolor{gray}{\# For the shuffler:}
        \STATE Set batch end-time: $t_m=t$
        \STATE Permute all received messages uniformly at random
        $Y_{m,1}=S_1\left(\lbrace M_{\tau,1} \rbrace_{t_{m-1}+1 \le \tau \le t_m}\right)$ and $Y_{m,2}=S_2\left(\lbrace M_{\tau,2} \rbrace_{t_{m-1}+1 \le \tau \le t_m}\right)$
        % \STATE Send uniformly permuted messages to the shuffler
        % \STATE  $Y_{m,1}=S_1\left(\lbrace M_{\tau,1} \rbrace_{t_{m-1}+1 \le \tau \le t_m}\right)$
        % \STATE $Y_{m,2}=S_2\left(\lbrace M_{\tau,2} \rbrace_{t_{m-1}+1 \le \tau \le t_m}\right)$
        \STATE \textcolor{gray}{\# For the analyzer (server):}
         \STATE Compute per-batch statistics $\widetilde{u}_m =A_2(Y_{m,1})$ and $\widetilde{V}_m =A_1(Y_{m,2})$ using shuffled messages
         \STATE Update overall batch statistics: $u_m = u_{m-1} + \widetilde{u}_m$, $V_m = V_{m-1} + \widetilde{V}_m$
         \STATE  Compute parameter estimate $\hat{\theta}_{m} = V_{m}^{-1} u_{m}$
         \STATE Send updated models $(\hat{\theta}_{m}, V_m)$ to users
         \STATE Increase batch counter: $m=m+1$ 
   \ENDIF
   \ENDFOR
\end{algorithmic}
\end{algorithm}

\section{A Shuffle Algorithm for Contextual Bandits}\label{sec:shuffleModel}
In this section, we introduce a general algorithmic framework (Algorithm~\ref{alg:BOFUL}) for linear contextual bandits under the shuffle model. We build on the celebrated LinUCB algorithm \cite{Chu2011,abbasi2011improved}, which is an application of the
\emph{optimism in the face of uncertainty} principle to linear bandits. Throughout the paper, we make the following assumptions, which are standard in the literature \cite{Chu2011,shariff2018differentially}.

\begin{assumption}[Boundedness]
\label{ass:bounded}
The rewards are bounded for all $t$, i.e., $y_t \in [0,1]$. Moreover, the parameter vector and the features have bounded norm, i.e., $\norm{\theta^*}_2 \leq 1$ and $\sup_{c,a}\norm{\phi(c,a)}_2 \leq 1$.\footnote{All terms are assumed to be bounded by one via normalization.}
\end{assumption}

% \xingyu{The shuffler in two cases also need care, and I will write them down first.}

\subsection{Algorithm: Shuffle Private LinUCB}
Our shuffle algorithm for contextual bandits consist of batches with a fixed size $B$, i.e., we have total $M = T/B$ batches.\footnote{We assume, wlog, total number of rounds $T$ is multiple of $B$.} The central idea is to construct, for each batch $m$, a $d$-dimensional ellipsoid $\cE_m$ with centre $\hat\theta_m$, shape matrix $V_m$ and radius $\beta_m$ so that it contains the unknown parameter $\theta^*$ with high probability. Moreover, the ellipsoids are designed while keeping the privacy setting in mind. They depend on the randomizer, shuffler and analyzer employed in the shuffle protocol based on required privacy levels $\epsilon,\delta$. The personal data of user $t$ in batch $m$ is given by the feature vector $\phi(c_t,a_t)$ and reward $y_t$, where the action $a_t$ is selected given the context $c_t$ as
\begin{align*}
    a_t \!\in\! \argmax_{a \in \cX} \{\! \inner{\phi(c_t,a)}{\hat{\theta}_{m\!-\!1}}\! +\! \beta_{m\!-\!1} \!\norm{\phi(c_t, a)}_{{V}_{m\!-\!1}^{-1}}\!\}.
\end{align*}
 % \xingyu{why we use $m'$ rather than $m-1$?}\sayak{To sound nice, maybe?}\xingyu{I think people maybe confused by these? Because in our algorithm, there is no $m'$ Anyway, I am fine. Just worry it may cause some confusion since later you also use $m'$ as an index, which is a little confused}
We consider a fixed randomizer across all the batches given by two functions $R_1$ and $R_2$ that locally operate on the vectors $\phi(c_t,a_t)y_t$ and matrices $\phi(c_t,a_t)\phi(c_t,a_t)^\top$, respectively. Similarly, we have shuffler functions $S_1$ and $S_2$ operating on batches (of size $B$) of those respective randomized messages. Finally, the analyzer functions $A_1$ and $A_2$
receive permuted messages from $S_1$ and $S_2$, and output, for each batch $m'$, an aggregate vector $\widetilde u_{m'}$ and matrix $\widetilde V_{m'}$, respectively. The central server uses this aggregate batch statistics to construct the ellipsoid: $V_m=\lambda I_d+\sum_{m'=1}^{m}\widetilde V_{m'}$ and $\hat\theta_m=V_m^{-1}\sum_{m'=1}^{M}\widetilde u_{m'}$. For a given confidence level $\alpha \in (0,1]$, the radius of the ellipsoid is set as
$\beta_m = O\left(\sqrt{2\log\left(\frac{2}{\alpha}\right) + d\log\left(1+\frac{t_m}{d\lambda}\right)} + \sqrt{\lambda}\right)$, where $t_m$ is the time when batch $m$ ends. The regularizer $\lambda$ and thus, in turn, the confidence radius $\beta_m$ typically depend on the total noise infused in the shuffle protocol.
% and complexities of the contextual bandit model (i.e., parameter norm and feature dimension). 
On a high level, these randomizer, shuffler and analyzer functions together provide suitable random perturbations to the Gram matrices and feature-reward vectors based on
the privacy budget $\epsilon,\delta$, and in turn,
they affect the regret performance via the noise levels of these perturbations. Next, we turn to discuss specific choices of these functions, and the associated performance guarantees of Algorithm~\ref{alg:BOFUL} under those choices.

\subsection{Achieving SDP via LDP Amplification}\label{sec:amplification}

In this section, we show that our general framework (Algorithm~\ref{alg:BOFUL}) enables us to directly utilize existing LDP mechanisms for linear contextual bandits to achieve a finer utility-privacy trade-off. The key idea here is to leverage the explicit privacy amplification property of the shuffle protocol \cite{feldman2020hiding}. Roughly, the privacy guarantee can be amplified by a factor of $\sqrt{B}$ by randomly permuting the output of an LDP mechanism independently operating on a batch of $B$ different users. In other words, the same level of privacy can be achieved for each user by adding a $\sqrt{B}$ factor \emph{less} noise in the presence of shuffler, yielding a better utility. Specifically, we instantiate Algorithm~\ref{alg:BOFUL} with the shuffle protocol $\mathcal{P}_{\text{Amp}} \!=\! (\cR_{\text{Amp}}, \cS_{\text{Amp}}, \cA_{\text{Amp}})$, where we employ standard Gaussian mechanism \cite{dwork2014algorithmic} as randomizer functions. Essentially, we inject independent Gaussian perturbation to each entry of the vector $\phi(c_t,a_t)y_t$ and the matrix $\phi(c_t,a_t)\phi(c_t,a_t)^{\top}$ with variances $\sigma_1^2$ and $\sigma_2^2$, respectively. We make sure the noisy matrix is symmetric by perturbing upper diagonal entries, and copying those to the lower terms. The noise variances are properly tuned depending on the sensitivity of these elements to achieve desired level of privacy. In this case, the shuffler functions simply permute its data uniformly at random, and the job of the analyzer is to simply add its received data (i.e., vectors or matrices). We defer further details on the protocol $\cP_{\text{Amp}}$ to Appendix~\ref{app:amp} and focus on performance guarantees first.

\begin{theorem}[Performance under LDP amplification]
% \xingyu{shall we replace regret by something performance? since we also have privacy? small issue anyway}\sayak{Ok with anything}\xingyu{you decide it:)}
\label{thm:amp-main}
Fix time horizon $T \!\in\! \Nat$, batch size $B \!\in\! [T]$, confidence level $\alpha \!\in\! (0,1]$, privacy budgets $\delta \!\in\! (0,1]$, $\epsilon \!\in\! (0,\sqrt{\frac{\log(2/\delta)}{B}}]$. Then, Algorithm~\ref{alg:BOFUL} instantiated using shuffle protocol $\cP_{\text{Amp}}$ with noise $\sigma_1\!=\!\sigma_2\!=\! \frac{4\sqrt{2\log(2.5B/\delta)\log(2/\delta)}}{\epsilon\sqrt{B}}$, and regularizer $\lambda \!=\! \Theta(\sqrt{T}\sigma_1 (\sqrt{d}\!+\! \sqrt{\log(T/B\alpha)})$, enjoys the regret
\begin{align*}
     \text{Reg}(T) \!=\! O\!\!\left(\!\!dB\log T \!+\! \frac{\log^{1/2}(B/\delta)}{\epsilon^{1/2}B^{1/4}} d^{3/4} T^{3/4}\log^2\!(T\!/\alpha)\!\!\right)\!,
\end{align*}
% \xingyu{there will be an additional $\log^{1/4}(2/\delta)$ in the second term. Tricky thing in the amp method.  Now it makes sense to me since in both protocols, the dependence on $\delta$ should be roughly $\log^{1/2}(1/\delta)$}\xingyu{To save space, instead of $\log^{1/4}(B/\delta) \log^{1/4}(\delta/2)$, I just write $\log^{1/2}(B/\delta)$.}
% \xingyu{there is an additional $\log T$ in the second term.One way is to use $\tilde{O}$ to remove all $\log T$. see Lemma A.4 }\sayak{$\log^2 T$? Let's keep logs to sound formal here.}\xingyu{$\log(T)$ and $\log(T/\alpha)$ is kind of different?}\sayak{Its fine.All log and polylogs}\xingyu{sure, we keep the more accurate one the appendix anyway.}\sayak{yes. This one is also accurate. But loose}\xingyu{yes, but loose}
with probability at least $1-\alpha$. Moreover, it satisfies $O(\epsilon,\delta)$-shuffle differential privacy (SDP). 
\end{theorem}
% \textbf{Cost of privacy.}\xingyu{remove this line?}

% \sayak{Moreover, setting $\lambda = \Theta(\max\{\sqrt{T}\sigma_1 (\sqrt{d}+ \sqrt{\log(M/\alpha)}),1\}$ and $\beta_m = O\left(\sqrt{2\log\left(\frac{1}{\alpha}\right) + d\log\left(1+\frac{t_m}{d\lambda}\right)} + \sqrt{\lambda}\right)$, it enjoys the regret bound

% \begin{corollary}[Best possible regret]
% Setting batch size $B = O(d^{-1/5}\epsilon^{-2/5}T^{3/5}(\log(T/\delta))^{1/5})$, we achieve regret 
% \begin{align*}
%     R(T) = \tilde{O}\left(d^{4/5} T^{3/5}\epsilon^{-2/5}\left(\log(T/\delta)\right)^{1/5}\right).
% \end{align*}
% \end{corollary}

\begin{corollary}\label{cor:best_reg_amp}
Setting batch size $B = O(T^{3/5})$ in Algorithm~\ref{alg:BOFUL}, we can achieve regret $\widetilde{O}\left( \frac{T^{3/5}}{\sqrt{\epsilon}}\log^{1/2}(T/\delta)\right)$.\footnote{Note that with a careful choice of $B$ (depending on privacy parameters $\epsilon,\delta$), we can have a better regret dependence on $\epsilon,\delta$. See Corollary~\ref{cor:amp_util} for details.}
% \xingyu{the $\log$ term is $\log(T/\delta)$ since $B$ is on the order of $T$}\xingyu{Due to this, we may not have the same terms in $\delta$ as in previous works. We can directly setting $B = T^{3/5}\epsilon^{-2/5}$ and the regret bound. In the following discussion, we can mention that our dependence on $\epsilon$ is better than previous works, and on $T$ is perfectly in-between. is this ok? I understand your point that setting $B = T^{3/5}$ allows us to easily catch the trade-off compared to previous works. But now, due to $\delta$ term, we cannot achieve it.}
\end{corollary}
% \xingyu{Shall we mention that a careful choice of $B$, can have a better dependence on $\epsilon$, and refer to Appendix~\ref{app:amp}, since here $B$ is independent of $\epsilon$. This is kind of important when people may look at facebook paper, and argue that their result is better than ours.}\sayak{YEs let's point to appendix. Bot not compare with them much.}\xingyu{agree, just point that we can do better when $B$ depends on $\epsilon$, that's it}

\paragraph{Comparsion with central and local DP models.} At this point, we turn to compare the regret of our Shuffle Private LinUCB algorithm to that of LinUCB under central model with JDP guarantee\footnote{JDP, or, joint differential privacy, is a notion of privacy under central trust model specific to contextual bandits. See Appendix~\ref{app:JDP}.}
\cite{shariff2018differentially} and local model with LDP \cite{zheng2020locally} guarantee.
As mentioned before, LinUCB achieves
$\tilde{O}\left(\sqrt{\frac{T}{\epsilon}}\right)$ and $\tilde{O}\left(\frac{T^{3/4}}{\sqrt{\epsilon}}\right)$ regret under JDP and LDP guarantees, respectively. As seen in Corollary \ref{cor:best_reg_amp}, our regret bound in the shuffle trust model lies perfectly in between these two extremes. Importantly, it improves over the $T^{3/4}$ scaling in the (stronger) local trust model, achieving a better trade-off between regret and privacy.
However, it couldn't achieve the optimal $\sqrt{T}$ scaling in the (weaker) central trust model. It remains an open question whether $\sqrt{T}$ regret can be achieved under any notion of privacy stronger than the central model. 
% \xingyu{Another remark is needed to say that we can achieve better trade-off without requiring the mixed batch algorithm as claimed to be necessary in Facebook paper.}\sayak{yes, one comparison can be here or concurrent work is also fine. BTW is JDP $\log^{1/4}(1/\delta)$. I thought its $\log^{1/2}$}\xingyu{it is $\log^{1/4}(1/\delta)$ for both LDP and JDP. You can easily verify it via our general regret bound. since each Gassuian is $\sqrt{\log}/\epsilon$}
\begin{remark}
Our shuffle protocol $\cP_{\text{Amp}}$, by design, provides a certain level of local privacy to each user. Specifically, for batch size $B$, Algorithm~\ref{alg:BOFUL} is $O(\epsilon\sqrt{B/\log(2/\delta)},\delta/B)$-LDP. Furthermore, since shuffe model ensures a higher level of trust than the central model, Algorithm~\ref{alg:BOFUL} is also $O(\epsilon,\delta)$-JDP. See Appendix~\ref{app:amp} for details.
\end{remark}

Apart from achieving a refined utility-privacy trade-off, the above shuffle protocol $\cP_{\text{Amp}}$ requires minimum modifications over existing LDP mechanisms. However, the privacy guarantee in Theorem~\ref{thm:amp-main} holds only for small privacy budget $\epsilon$ particularly when the batch size $B$ is large, which could potentially limit its application in some practical scenarios (e.g., when $\epsilon$ is around $1$ or larger~\cite{apple}). Moreover, $\cP_{\text{Amp}}$ needs to communicate and shuffle \emph{real} vectors and matrices, which
are often difficult to encode on finite computers in practice~\cite{canonne2020discrete,kairouz2021distributed} and a naive use of finite precision approximation may lead to a possible failure of privacy protection~\cite{mironov2012significance}. To overcome these limitations of $\cP_{\text{Amp}}$, we introduce a different instantiation of Algorithm~\ref{alg:BOFUL} in the next section. 

% also enjoys the benefit of simple implementation, i.e., a minimum amount of modifications on existing LDP algorithm. However, there also exist two fundamental limitations.

% First, the SDP guarantee in Theorem~\ref{thm:amp-main} is only attained for a sufficiently small $\epsilon$ when $B$ is large,  where the privacy budget $\epsilon$ is often set to around $1$ or larger []. Second, users under $\cP_{\text{Amp}}$ would send messages consisting of real vectors or matrices, which is often difficult in practice []. 

\subsection{Achieving SDP via Vector Summation}\label{sec:vector-sum}

We instantiate Algorithm 1 with the shuffle protocol $\cP_{\text{Vec}} \!=\! (\cR_{\text{Vec}}, \cS_{\text{Vec}}, \cA_{\text{Vec}})$, where we rely on a
particularly efficient and accurate mechanism for summing vectors with bounded $\ell_2$ norms \cite{cheu2021shuffle}. First, the local randomizer $\cR_{\text{Vec}}$ adopts a one-dimensional randomizer that operates independently on each entry of the vector $\phi(c_t,a_t)y_t$ and the matrix $\phi(c_t,a_t)\phi(c_t,a_t)^{\top}$, respectively.
This adopted one-dimensional randomizer transmits only bits ($0/1$) via a fixed-point encoding scheme~\cite{cheu2019distributed}, and ensures privacy by injecting binomial noise. In particular, given any entry $x \!\in\! [0,1]$, it is first encoded as $\hat{x} \!=\! \bar{x} \!+\! \gamma_1$, using an accuracy parameter $g \!\in\! \Nat$, where $\bar{x} \!=\! \lfloor{x g}\rfloor$ and $\gamma_1 \!\sim\! \texttt{Ber}(xg -\bar{x})$. Then a binomial noise is generated, $\gamma_2 \!\sim\! \texttt{Bin}(b,p)$, where parameters $b \!\in\! \Nat, p\!\in\!(0,1)$ control the privacy noise. The output of the one-dimensional randomizer is simply a collection of total $g+b$ bits, in which $\hat{x} \!+\! \gamma_2$ bits are $1$ and the rest are $0$. Combining the outputs of the one-dimensional randomizer for each entry of vector $\phi(c_t,a_t)y_t$ and matrix $\phi(c_t,a_t)\phi(c_t,a_t)^{\top}$, yield final outputs of randomizer. The shuffler functions in $\cS_{\text{Vec}}$ simply permutes all the received bits uniformly at random. The job of the analyzer $\cA_{\text{Vec}}$ is to add the received bits for each entry, and remove the bias introduced due to encoding and binomial noise. This is possible since bits are already labeled entry-wise when leaving $\cR_{\text{Vec}}$. The constants $g,b,p$ are left as tunable parameters of $\cP_{\text{Vec}}$, and need to be set properly depending on the desired level of privacy. The detailed implementation of this scheme is deferred to Appendix~\ref{sec:app_vec}.  The following theorem states the performance guarantees of Algorithm~\ref{alg:BOFUL} instantiated with $\cP_{\text{Vec}}$.

\begin{theorem}[Performance under vector sum]
\label{thm:vec-main}
Fix batch size $B \!\in\! [T]$, privacy budgets $\epsilon \!\in\! (0,15]$, $\delta \!\in\! (0,1/2)$.
Then, Algorithm~\ref{alg:BOFUL} instantiated with $\cP_{\text{Vec}}$ with parameters $p \!=\! 1/4$, $g \!=\! \max\{2\sqrt{B}, d, 4\}$ and $b \!=\! \frac{C\cdot g^2\cdot \log^2\left(d^2/\delta\right)}{\epsilon^2B}$ is $(\epsilon,\delta)$-SDP, where $C >> 1$ is some sufficiently large constant. 
% there are choices of parameters $g,b \!\in\! \Nat$ and $p \!\in\! (0,1/2)$ depending on $B,\epsilon,\delta$ and feature dimension $d$ such that Algorithm~\ref{alg:BOFUL} instantiated with $\cP_{\text{Vec}}$ is $(\epsilon,\delta)$-SDP. 
 %  let $p = 1/4$,
% \begin{align*}
%     g = \max\{2\sqrt{B}, d, 4\},  b = \frac{24\cdot 10^4\cdot g^2\cdot \left(\log\left(\frac{4\cdot(d^2+1)}{\delta}\right)\right)^2}{\epsilon^2B}.
% \end{align*}
Furthermore, for any $\alpha \!\in\! (0,1]$, setting $\lambda \!=\! \Theta\!\left(\! \frac{\log(d^2/\delta)\sqrt{T}}{\epsilon\sqrt{B}} (\!\sqrt{d}\!+\! \sqrt{\log(T/B\alpha)}\!\right)$, it enjoys the regret 
\begin{align*}
     \text{Reg}(T) \!=\! O\!\left(\!\!dB\log T \!+\! \frac{\log^{1/2}(d^2/\delta)}{\epsilon^{1/2}B^{1/4}} d^{3/4} T^{3/4}\log^2\!(T\!/\alpha)\!\!\right)\!,
\end{align*}
with probability at least $1-\alpha$.
\end{theorem}
\begin{remark}
Similar to Corollary~\ref{cor:best_reg_amp}, an $\widetilde O\left(\frac{T^{3/5}}{\sqrt{\epsilon}}\right)$ regret can also be achieved in this case by setting $B\!=\!O(T^{3/5})$, but the dependence on $\delta$ is now: $\log^{1/2}(d^2/\delta)$ as compared to $\log^{1/2}(T/\delta)$.
Moreover, in contrast to Theorem~\ref{thm:amp-main}, the guarantees hold for a wide range of $\epsilon$, making $\cP_{\text{Vec}}$ better suitable for practical purposes ~\cite{apple}. Finally, as before, if $B$ also depends on privacy parameters, the dependence on $\epsilon, \delta$ can be improved, see Corollary~\ref{cor:vec_util}.
\end{remark}

\begin{remark}
$\cP_{\text{Vec}}$ can also be regarded as privacy amplification of Binomial mechanism (rather than Gaussian mechanism in $\cP_{\text{Amp}}$), which is the reason that it also offers a certain degree, $O(\epsilon\sqrt{B},\delta)$, to be precise, of LDP guarantee.
\end{remark}

\begin{remark}
Both shuffle protocols, $\cP_{\text{Amp}}$ and $\cP_{\text{Vec}}$, in fact, can be tuned to satisfy $(\epsilon,\delta)$-LDP by sacrificing on regret performance. See Corollaries \ref{cor:amp_priv} and \ref{cor:vec_priv} for details.
\end{remark}

% Setting $\sigma_1 = \sigma_2 = \frac{4\sqrt{2\log(2.5/\delta_0)}}{\epsilon_0}$, we can show that $\cP_{\text{Amp}}$ is $(\epsilon_0,\delta_0)$-LDP. Further suppose $B = O(T^{3/4})$, then 
% $\cP_{\text{Amp}}$ achieves LDP regret $ \tilde{O}\left(T^{3/4}\epsilon_0^{-1/2}\right).$ Simultaneously, $\cP_{\text{Amp}}$ achieves $O(\epsilon_0/\sqrt{B},\delta_0 B)$-SDP and JDP. Similarly, Setting $g,b,p$ depending on $\epsilon_0,\delta_0$, we can show that $\cP_{\text{Vec}}$ is $(\epsilon_0,\delta_0)$-LDP. Further suppose $B = O(T^{3/4})$, then 
% $\cP_{\text{Vec}}$ achieves regret $ \tilde{O}\left(T^{3/4}\epsilon_0^{-1/2}\right).$ Simultaneously, $\cP_{\text{Vec}}$ achieves $O(\epsilon_0/\sqrt{B},\delta_0)$-SDP and JDP

\subsection{Key Techniques: Overview}

In this section, we provide a generic template of regret bound for linear contextual bandits under the shuffle model of privacy. To this end, we need following notations to discuss the effect of noise added by shuffle protocol, in the learning process. Let $n_m \!=\! \widetilde{u}_m \!-\! \sum_{t = t_{m-1}+1}^{t_m} \phi(c_t, a_t)y_t$ and $N_m \!=\! \widetilde{V}_m \!-\! \sum_{t = t_{m-1}+1}^{t_m}\phi(c_t,a_t)\phi(c_t,a_t)^{\top}$ denote the total noise added during batch $m$ in the feature-reward vector, and in the Gram-matrix, respectively. Furthermore, assume that there exist constants $\widetilde{\sigma}_1$ and $\widetilde{\sigma}_2$ such that for each batch $m$, (i) $\sum_{m'=1}^m n_{m'}$ is a random vector whose entries are independent, mean zero, sub-Gaussian with variance at most $\widetilde{\sigma}_1^2$, and
(ii) $\sum_{m'=1}^m N_{m'}$ is a random symmetric sub-Gaussian matrix whose entries on and above the diagonal are independent with variance at most $\widetilde{\sigma}_2^2$. Let $\sigma^2 \!=\! \max\lbrace \widetilde \sigma_1^2,\widetilde \sigma_2^2\rbrace$.
Then, we have the following result.

% We first present a generic regret bound of Algorithm~\ref{alg:BOFUL} when the added private noise is sub-Gaussian, which is summarized in the following assumption.

% \begin{theorem}[Informal]
% \label{cor:subG-main}
% Let Assumption~\ref{ass:bounded} and Assumption~\ref{ass:subG-main}  hold and $\sigma = \max\{\widetilde{\sigma}_1, \widetilde{\sigma}_2\}$.
% % where and $\beta_m = O\left(\sqrt{2\log\left(\frac{1}{\alpha}\right) + d\log\left(1+\frac{t_m}{d\lambda}\right)} + \sqrt{\lambda}\right)$.
% Then, for any $\alpha \in (0,1]$, let  $\lambda = \Theta(\max\{\sigma (\sqrt{d}+ \sqrt{\log(M/\alpha)}),1\}$, with probability at least $1-2\alpha$, the regret of Algorithm~\ref{alg:BOFUL} satisfies 
% \begin{align*}
%     R(T) = &O\left({dB}\log T + d\sqrt{T}\log(T/\alpha)\right) \\
%     &+ O\left(\sqrt{\sigma T}d^{3/4}\log T \log (T/\alpha)\right).
% \end{align*}
% \end{theorem}

% $\lambda \approx \Theta(\max\{\sigma (\sqrt{d}+ \sqrt{\log(M/\alpha)}),1\}$

\begin{lemma}[Informal]
\label{lem:subG-main}
With the choice of $\lambda \approx \sigma(\sqrt{d}+\sqrt{\log(T/(B\alpha))})$, the regret of Algorithm~\ref{alg:BOFUL} satisfies
\begin{align*}
    \text{Reg}(T) \!=\! \widetilde{O}\!\left(\!{dB}\!+\! d\sqrt{T}\!+\!\sqrt{\sigma T}d^{3/4}\!\right)\,\text{with high probability.}
\end{align*}
\end{lemma}
With the above result, one only needs to determine the noise variance $\sigma^2$ under different privacy protocols. We illustrate this with the shuffle protocols introduced in previous sections. First, note that since we assume unique users, Algorithm~\ref{alg:BOFUL} is SDP if each batch is SDP. Now, for the LDP amplification protocol $\cP_{\text{Amp}}$, in order to guarantee SDP for each batch with sufficiently small privacy loss $\epsilon$, it suffices to work with an LDP mechanism with loss $\epsilon\sqrt{B}$ by virtue of  amplification.\footnote{We provide intuition without worrying about the details related to $\delta$-dependent terms. Refer to Appendix~\ref{app:unified_regret} for formal proofs.} We ensure this by choosing Gaussian mechanism with noise variance $O(1/(\epsilon^2B))$. Hence, the total noise variance added by $\cP_{\text{Amp}}$ is $\sigma^2 \approx O(\frac{T}{\epsilon^2 B})$. Thus, by Lemma~\ref{lem:subG-main}, we obtain the result in Theorem~\ref{thm:amp-main}. Similarly, for the vector sum protocol $\cP_{\text{Vec}}$, we ensure $\cP_{\text{Vec}}$ to be SDP by properly setting parameters $g, b, p$. Moreover, the analyzer's outputs are unbiased estimates of the sum of the non-private vectors (matrices) within that batch, and the entry-wise private noise is sub-Gaussian with variance of $O(\frac{1}{\epsilon^2})$. 
Thus, the total noise variance added by $\cP_{\text{Vec}}$ is $\sigma^2 \approx O(\frac{T}{\epsilon^2B})$, and hence, by Lemma~\ref{lem:subG-main}, we have the result in Theorem~\ref{thm:vec-main}.

% use its property that with a proper choice of $g$ and $b$, to guarantee $(\epsilon,\delta)$-SDP, its added noise at each batch is sub-Gaussian with variance $O(1/\epsilon^2)$. Thus, the total noise $\sigma^2 \approx O(\frac{T}{\epsilon^2B})$, i.e., sum of $M = T/B$ i.i.d noise. By Theorem~\ref{ass:subG-main}, we have the same regret as in~\eqref{eq:regret_intuition}. 
% \begin{align}
% \label{eq:regret_intuition}
%     R(T) = \tilde{O}(dB + d^{3/4}T^{3/4}B^{-1/4}\epsilon^{-1/2}).
% \end{align}

% \begin{remark}
% For the standard local model in~\cite{zheng2020locally}, we have $B = 1$ and $\sigma^2 \approx O(T\frac{1}{\epsilon^2})$ since there are total $O(T)$ Gaussian noise. On the other hand, in the standard central model~\cite{shariff2018differentially}, $B = 1$ and $\sigma^2 \approx 
% O(\frac{1}{\epsilon^2}\log T)$ since there are at most $\log(T)$ added Gaussian noise due to the use of tree-based algorithm
% \end{remark}

\begin{remark}
Lemma~\ref{lem:subG-main}, in fact, can serve as a general template of regret for private linear contextual bandit algorithms. For example, for the local model~\cite{zheng2020locally}, $B \!=\! 1$ and $\sigma^2 \!\approx\! \frac{T}{\epsilon^2}$, yielding $\widetilde O\left(\frac{T^{3/4}}{\sqrt{\epsilon}}\right)$ regret. Similarly, for the central model~\cite{shariff2018differentially}, $B \!=\! 1$ and $\sigma^2 \!\approx\! 
\frac{\log T}{\epsilon^2}$, which yields $\widetilde O\left(\frac{T^{1/2}}{\sqrt{\epsilon}}\right)$ regret.
\end{remark}

\section{Regret Performance under Returning Users}\label{sec:return}

Similar to existing work on differentially private bandits, in previous sections, we have assumed that all participating users are unique, i.e., each user participates in the protocol only at one round. A more practical scenario is that an user can contribute with her data at multiple rounds. For example, consider the mobile medical application described in the introduction. The cloud server can collect one particular user's data during multiple batches to track the effectiveness of its treatment plan over a period, and hence, use same user's data multiple times to update its recommendation algorithm. Motivated by this, we provide privacy and regret guarantees of Algorithm~\ref{alg:BOFUL} under the setting of \emph{returning users} in linear contextual bandits. We first define the setting of returning users that we consider in this section, and then state the performance guarantee for Algorithm~\ref{alg:BOFUL}.

% \xingyu{I now tend to feel that this is kind of misleading in our previous notations and wording. That is we first fix $M_0$ and then tune $B$ (say if $M_0=1$, we choose $B = T$? even contradicts with our previous unique users in some sense).  This is not used in practice. What we used in practice is that we first fix batch size $B$ in advance, and then we chosse how to sample users as decribed in the above example. 
% % Under our assumption of returning users, we have the following guarantee. 
% We are not trying to create batch based on how frequently users contribute. Rather, it is the server to decide how to sample users across different batches.}

% \xingyu{I create the current one in new\_return.tex file and  keep the old one in old\_return.tex unchanged} 

\begin{assumption}[Returning Users]
\label{def:return_def}
For a given time horizon $T \in \Nat$ and batch size $B \in [T]$, any user can participate in \emph{all} $M = T/B$ batches, but within each batch $m \in [M]$, she only contributes once.
\end{assumption}
In addition to the above motivating example, this assumption also captures many practical adaptive learning scenarios such as clinical trials and product recommendations, in which each trial (batch) involves a group of unique people, but the same person may participate in multiple trials~\cite{ren2020batched,schwartz2017customer}.

% \sayak{This does not look like a definition.}\xingyu{I prefer to call it assumption actually.}

% \begin{assumption}[Returning Users]
% \label{ass:return_ass}
% We assume that each user can participate in multiple batches, but within each batch she only contributes once. 
% \end{assumption}
% \sayak{Can we state this inline--not making any assumptions. Just say that each user are now allowed to participate in multiple batches.}\xingyu{I am fine with this:) but we need to emphasize that the user cannot appear twice in the same batch.}
% \sayak{Yes} \xingyu{just go ahead:)}\sayak{I will do this section and exps tomoroo} 
% \xingyu{sure, I will finish the polish for section 3, i.e, the appendix}great \xingyu{I like the current style of section 3.}\sayak{Yes, it looks good. Section 4 is very important for us. Thanks.}\xingyu{while you edit it, you can also have a double check:) haha. I give the intuition for user-level JDP, which you can easily understand. For SDP, it directly follows from composition, no fancy things.}great
% \sayak{Some equation refs are not working I see}
% \xingyu{I will fix them. You just edit your comments. Finally, I will go over to make sure there no warnings of latex.}okay

\begin{theorem}[Performance guarantees (informal)]
\label{thm:return-main}
Under Assumption~\ref{def:return_def}, we obtain the following results for $\cP_{\text{Amp}}$ and $\cP_{\text{Vec}}$, respectively. 
% \xingyu{we get something like:}
% \begin{align*}
%     R(T) &= dB + \frac{M_0^{1/4}T^{3/4}d^{3/4}}{\epsilon^{1/2} B^{1/4}}\\
%     & = d\frac{T}{M} + d^{3/4}\frac{\sqrt{M_0 T}}{\epsilon^{1/2}} \left(\frac{M}{M_0}\right)^{1/4}
% \end{align*}
% \xingyu{what do you think, Sayak? This one seems not that perfect.}\sayak{for $M_0=1$, we recover Thm 3.2. If $M_0 =M$, we match JDP, Yes, same order. Basically, setting $B=T/M_0$, we match JDP. call?ye}\xingyu{but if $M_0 = 1$, this not true. we cannot set $B$ according to $M_0$. $B$ should not depend on $M_0$. sure zoom} \xingyu{yes. as expected. or $M_0$ same order as $M$, we are fine, right? Not has to be equal, same order is fine? just some constant. Ok, let's use this results?}
% % \xingyu{it seems that this is not clear. In some sense, only in the worst case, it matches central model, i.e., $M_0 = T/B$?}

(i) For any $\epsilon \!\leq\! \frac{2}{B}\log(2/\delta)\sqrt{2T}$ and $\delta \!\in\! (0,1]$, Algorithm~\ref{alg:BOFUL} instantiated using $\cP_{\text{Amp}}$ with noise levels $\sigma_1 \!=\! \sigma_2 \!=\! \frac{16\log(2/\delta)\sqrt{T(\log(5T/\delta))}}{\epsilon{B}}$ is $O(\epsilon,\delta)$-SDP, and enjoys, with high-probability, the regret bound
\begin{align*}
     \text{Reg}(T) = \widetilde{O}\left(\frac{dT}{M} + \sqrt{\frac{MT}{\epsilon}}d^{3/4} \log^{3/4} (T/\delta)\right).
\end{align*}
(ii) For any $\epsilon \!\leq\! 15$, $\delta \!\in\! (0,1/2)$, there exist choices of parameters $g,b \!\in\! \Nat$, $p \!\in\! (0,1/2)$ depending on $B,\epsilon,\delta$ such that Algorithm~\ref{alg:BOFUL} instantiated using $\cP_{\text{Vec}}$ is $(\epsilon,\delta)$-SDP, and, enjoys, with high probability, the regret bound
\begin{align*}
     \text{Reg}(T) \!=\! \widetilde{O}\left(\frac{dT}{M} \!+\!\sqrt{\frac{MT}{\epsilon}}d^{3/4} \log^{3/4} (d^2M/\delta)\right).
\end{align*}

% \begin{align*}
%      \text{Reg}(T) \!=\! \widetilde{O}\left(\frac{dT}{M} \!+\!\frac{\log^{3/4}(d^2M/\delta)}{\epsilon^{1/2}} d^{3/4}\!\sqrt{MT}\right).
% \end{align*}
% Let $B = O(d^{-1/6}\epsilon^{-1/3}T^{2/3}(\log(Td^2/\delta))^{1/2})$, $g = \max\{2\sqrt{B}, d, 4\}$, $p=1/4$ and 
% \begin{align*}
%     b = \frac{ 10^7\cdot \log(2/\delta)\cdot g^2 \cdot T\cdot \left(\log\left(\frac{8\cdot T(d^2+1)}{B\delta}\right)\right)^2}{\epsilon^2B^2}.
% \end{align*}
\end{theorem}

% \begin{theorem}[Performance guarantees (informal)]
% \label{thm:return-main}
% % Let Assumption~\ref{ass:bounded} and Assumption~\ref{ass:return_ass} hold. 
% We have the following results for $\cP_{\text{Amp}}$ and $\cP_{\text{Vec}}$, respectively. 

% (i) Fix any $B \!\in\! \Nat$, $\epsilon \!\leq\! \frac{2}{B}\sqrt{2T\log(2/\delta)}$ and $\delta \!\in\! (0,1]$. Then, Algorithm~\ref{alg:BOFUL} instantiated with $\cP_{\text{Amp}}$ with noise levels $\sigma_1 \!=\! \sigma_2 \!=\! \frac{16\sqrt{T\log(2/\delta)(\log(5T/\delta))}}{\epsilon{B}}$is $O(\epsilon,\delta)$-SDP, and enjoys, with high-probability, the regret bound
% \begin{align*}
%      \text{Reg}(T) = \widetilde{O}\left(dB\log T + \frac{\log^{1/2}(T/\delta)}{\epsilon^{1/2}B^{1/2}} d^{3/4} T\right)~.
% \end{align*}

% (ii) Fix any $B\!\in\! \Nat$, any $\epsilon \!\le\! 15$, and $\delta \!\in\! (0,1/2)$. There exist choices of parameters $g,b \!\in\! \Nat$ and $p \!\in\! (0,1/2)$ depending on $B,\epsilon,\delta$ such that Algorithm~\ref{alg:BOFUL} instantiated with $\cP_{\text{Vec}}$ is $(\epsilon,\delta)$-SDP, and, enjoys, with high probability, the regret
% \begin{align*}
%      \text{Reg}(T) = \widetilde{O}\left(dB\log T +\frac{\log^{3/4}(d^2T/(B\delta))}{\epsilon^{1/2}B^{1/2}} d^{3/4}T\right).
% \end{align*}
% % Let $B = O(d^{-1/6}\epsilon^{-1/3}T^{2/3}(\log(Td^2/\delta))^{1/2})$, $g = \max\{2\sqrt{B}, d, 4\}$, $p=1/4$ and 
% % \begin{align*}
% %     b = \frac{ 10^7\cdot \log(2/\delta)\cdot g^2 \cdot T\cdot \left(\log\left(\frac{8\cdot T(d^2+1)}{B\delta}\right)\right)^2}{\epsilon^2B^2}.
% % \end{align*}
% \end{theorem}

\begin{proof}[Proof sketch]
In contrast to Section~\ref{sec:shuffleModel} for unique users, where $(\epsilon,\delta)$-SDP guarantee for Algorithm~\ref{alg:BOFUL} can be established by showing each batch is $(\epsilon,\delta)$-SDP, we now need to guarantee that outputs of all the batches together have a total privacy loss of $(\epsilon,\delta)$.
This is due to the fact that now each batch can potentially operate on same set of users, and hence, one need to use advanced decomposition to calculate the total privacy loss. This leads to scaling up the noise variance by a multiplicative factor of $O(M)$ at each batch, which eventually leads to the above bound (the additional $M$ factor in $\delta$ also comes from advance composition).
\end{proof}
Interestingly, the privacy ($\epsilon,\delta$)-dependent term in above regret bounds match the one that can be achieved in the \emph{user-level} central trust model that handles returning users. Note that, since existing work in the central model of privacy (i.e., under JDP guarantee) assume unique users~\cite{shariff2018differentially}, we first generalize it to handle returning users. This can be viewed as the same form of generalization from \emph{event-level} DP to \emph{user-level} DP under continual observation, where the adjacent relation between two data streams changes from the flip of one single round to the flip of multiple rounds associated with a single user \cite{dwork2010differential}.\footnote{See Appendix~\ref{app:JDP} for formal definitions of event-level and user-level joint differential privacy (JDP).} As in standard notion of DP, one straightforward approach for converting event-level JDP to user-level JDP is to use group privacy~\cite{dwork2014algorithmic}. However, this black-box approach would blow up the terms dependent on $\delta$. To overcome this, we propose a simple modification of original (event-level) algorithm in~\cite{shariff2018differentially} so that it can handle returning users. In particular, user-level JDP can be achieved by scaling up the noise variance by a multiplicative factor of $M_0^2$, if any user participates in at most $M_0$ rounds. This follows from the fact that flipping one user now would change the $\ell_2$ sensitivity of the expanded binary-tree nodes from $O(\sqrt{\log T})$ to $O(M_0\sqrt{\log T})$. Note that we use $M_0$ to distinguish from the number of batches $M$ since there is no batch concept in standard central model. 
This modified version enjoys the following regret guarantee.

% \sayak{State what is the modification in short.}\xingyu{This is just stated in the theorem statement, i.e., reduce $\epsilon$ to $\epsilon/M$.}\sayak{Does that mean noise level is $\frac{M^2\log T}{\epsilon^2}$}\xingyu{yes. basically reduce $\epsilon$ in previous section by a factor of $M$. see the proof sketch below.}

% this modified algorithm attains nearly the same regret bound as in Theorem~\ref{thm:return-main}  under the same value of $B$, i.e., the same number of rounds one user can contribute. 

\begin{proposition}\label{prop:JDP-return}
If any user participates in at most $M_0$ rounds, the algorithm in \cite{shariff2018differentially}, with the above modification to handle user-level privacy, achieves the high-probability regret bound
\begin{align*}
    \text{Reg}(T) = \widetilde{O}\left(d\sqrt{T} + \sqrt{\frac{M_0T}{\epsilon}}d^{3/4} \log^{1/4} (1/\delta) \right).
\end{align*}
\end{proposition}

% with noise variance scaled up by a multiplicative factor of $M^2$ is able to achieve $(\epsilon,\delta)$-JDP (user-level) with a high-probability regret bound 
% \begin{align*}
%     \text{Reg}(T) = \widetilde{O}\left(d\sqrt{T} + \sqrt{\frac{MT}{\epsilon}}d^{3/4} \log^{1/4} (1/\delta) \right).
% \end{align*}

% \begin{proof}[Proof sketch]
% This follows from the fact that flipping one user now would change the $\ell_2$ sensitivity of the expanded binary-tree nodes from $O(\sqrt{\log T})$ to $O(M\sqrt{\log T})$. 
% \end{proof}
% \sayak{Why regret is $d^{3/4}$? ? we should give linear in $d$ term as well.}\xingyu{no, the same as previous CDP. you just check the CDP paper, you can get it. or directly look at our general result. see Lemma 3.9}\sayak{Please make the changes.}\xingyu{which change? in previous sections, you did not write any $d$ term I think.} I changed in the proposition
% \begin{remark}
% Setting $M = T^{1/3}$ yields the regret bound (i.e., $\tilde{O}(d^{5/6}T^{2/3}\epsilon^{-1/3})$) up to logarithmic factor in terms of $1/\delta$ and $T$.
% \end{remark}

\begin{remark}
Comparing Theorem~\ref{thm:return-main} and Proposition~\ref{prop:JDP-return}, we observe that the cost of privacy in the shuffle model is essentially same (upto a log factor) as in the central model under the setting of returning users.
In particular, if $M = M_0=T^{1/3}$ rounds (i.e., the same number of possible returning rounds for any user), the regret is $\widetilde O\left(\frac{T^{2/3}}{\sqrt{\epsilon}}\right)$ in both shuffle and user-level central trust models. 
See Appendix~\ref{app:return} for complete proofs and more details.
% For batch size $B \!=\! O(T^{2/3})$, i.e., when each user is allowed to participate in at most $T^{1/3}$ batches, the regret of Algorithm~\ref{alg:BOFUL} is $\widetilde{O}\left( \frac{T^{2/3}}{\sqrt{\epsilon}}\right)$ for both shuffle protocols $\cP_{\text{Amp}}$ and $\cP_{\text{Vec}}$ under the setting of returning users. As before, if $B$ also depends on $\epsilon$, then we can achieve a better regret $\widetilde{O}(T^{2/3}\epsilon^{-1/3})$, see Appendix~\ref{app:return}.
\end{remark}

\section{Simulation Results}

\begin{figure}[t]\centering
\begin{subfigure}[t]{.3\linewidth}
		\centering
			\includegraphics[width = 2.1in]{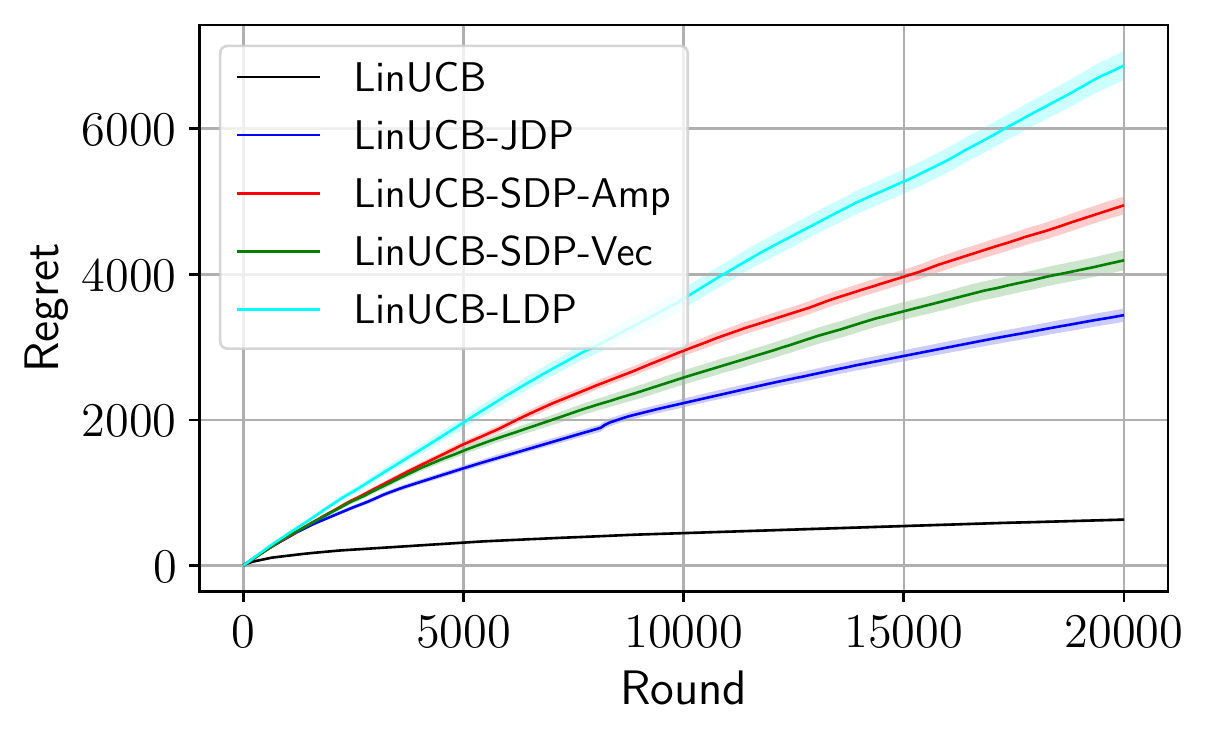}
			\caption{$\epsilon = 0.2$}
		\end{subfigure} \ \
		\begin{subfigure}[t]{.3\linewidth}
		\centering
			\includegraphics[width = 2.1in]{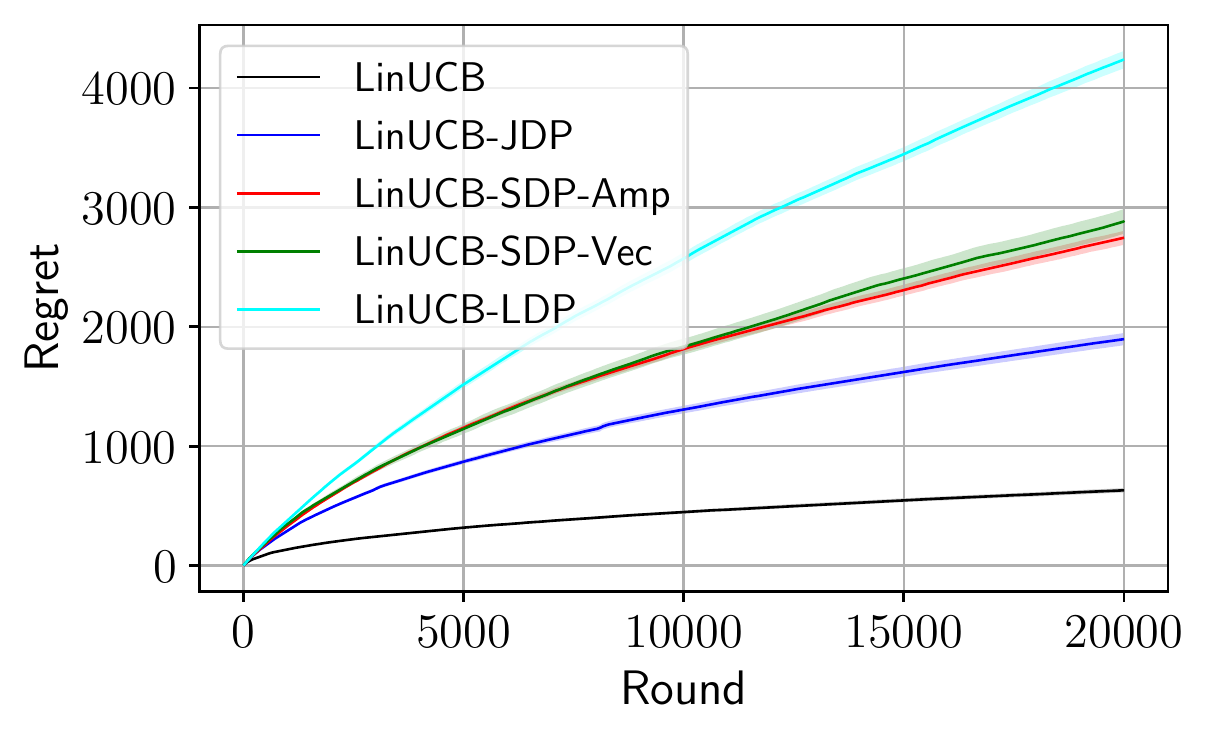}
			\caption{$\epsilon = 1$}
		\end{subfigure} \ \
		\begin{subfigure}[t]{.3\linewidth}
		\centering
			\includegraphics[width = 2.1in]{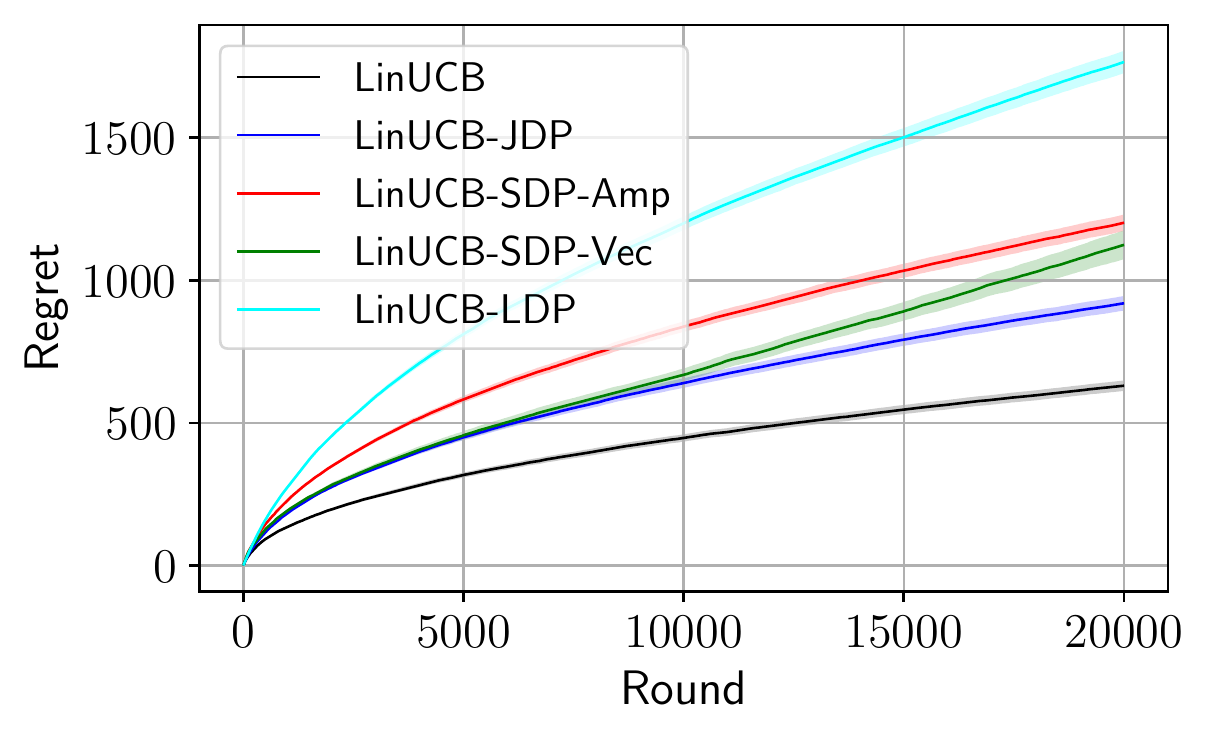}
			\caption{$\epsilon = 10$}
		\end{subfigure}
		 \vspace{-1mm}
		\caption{\footnotesize{Comparison of cumulative regret for LinUCB (non-private), LinUCB-JDP (central model), LinUCB-SDP (shuffle model) and LinUCB-LDP (local model) with varying privacy level $\epsilon=0.2$ (a), $\epsilon=1$ (b) and $\epsilon=10$ (c). For $\epsilon=0.2$ (higher privacy level), gap between private and non-private regret is higher as compared to $\epsilon=10$ (lower privacy level). In all cases, regret of LinUCB-SDP lies perfectly in between LinUCB-JDP and LinUCB-LDP, achieving finer regret-privacy trade-off.}  }\label{fig:all_algos}
		\vspace{0mm}
\end{figure}

In this section, we empirically evaluate the regret performance of Algorithm~\ref{alg:BOFUL} (under shuffle model), which we abbreviate as LinUCB-SDP-Amp and LinUCB-SDP-Vec when instantiated with $\cP_{\text{Amp}}$ and $\cP_{\text{Vec}}$, respectively. We compare them with the algorithms of \cite{shariff2018differentially} and \cite{zheng2020locally} under central and local models, which we call LinUCB-JDP and LinUCB-LDP, respectively. We benchmark these against the non-private algorithm of \cite{abbasi2011improved}, henceforth referred as LinUCB. For all the experiments, we consider $100$ arms, set $T = 20000$ rounds, and average our results
over 50 randomly generated bandit instances. Each instance is characterized by an
(unknown) parameter $\theta^*$ and feature vectors of dimension $d=5$. To ensure boundedness, similar to~\cite{vaswani2020old}, 
we generate each $\theta^*$ and feature vectors by sampling a $(d\!-\!1)$-dimensional vectors of norm $1/\sqrt{2}$ uniformly at random, and append it with a $1/\sqrt{2}$ entry. We consider Bernoulli $\lbrace 0,1\rbrace$ rewards.
We fix $\delta\!=\!0.1$ and plot the results for varying privacy level $\epsilon \in \lbrace 0.2,1,10\rbrace$ in Figure~\ref{fig:all_algos}. We use Batchsize $B=20$ for LinUCB-SDP. We postpone the results for $d=10,15$ to Appendix~\ref{app:sim}.
From Figure~\ref{fig:all_algos}, we observe that the regret performance of LinUCB-SDP (under both shuffle protocols $\cP_{\text{Amp}}$ and $\cP_{\text{Vec}}$) is indeed better than LinUCB-LDP. In addition, it is not surprising that LinUCB-SDP incurs a larger regret than
LinUCB-JDP. Moreover, the regret performance of LinUCB-SDP (in fact for any private algorithm) comes closer to that of LinUCB as $\epsilon$ increases, i.e, as the privacy guarantee becomes weaker. The experimental findings are
consistent with our theoretical results.

\section{Concluding Remarks}
\label{sec:conclude}
We conclude by discussing some important theoretical and practical aspects about shuffle protocols, and in general, about privacy in linear contextual bandits.

\textbf{Communications.}
In the protocol $\cP_{\text{Amp}}$, each participating user at each round need to send one $d$-dimensional real vector and one $d\times d$ real matrix.  On the other hand, the protocol $\cP_{\text{Vec}}$ only communicates $0/1$ bits. In particular, each participating user at each round sends out a total of $O(d^2(g+b))$ bits, where $g+b \approx \sqrt{B} + \log(1/\delta)/\epsilon^2$. Hence, $\cP_{\text{Vec}}$ might be more feasible in practice than $\cP_{\text{Amp}}$.

\textbf{Batched algorithms for local and central models.} Existing work on differentially private linear contextual bandits under both local and central models perform sequential update, i.e., the model estimates are updated after each round. As mentioned before, this may not be feasible in practice due to computational load. Fortunately, our proposed algorithm (Algorithm~\ref{alg:BOFUL}) along with its generic regret bound (Lemma~\ref{lem:subG-main}) also offers a simple way to design and analyze private algorithms for local and central models with batched update. In particular, we show that it suffices to update after every $B=\widetilde{O}(T^{3/4})$ rounds to achieve the same privacy-regret trade-off as in the sequential local model and every $B=\widetilde{O}(\sqrt{T})$ to match the sequential central model. See Appendix~\ref{sec:batched} for the details.

\textbf{Adaptive model update.}
One might wonder whether we can further reduce the update frequency to $O(\log T)$ via an adaptive model update schedule based on the standard determinant trick (Lemma 12 of~\cite{abbasi2011improved}). In this approach, the key step is to establish that $\norm{\phi(c,a)}_{V_{\tau_t}^{-1}} \!\leq\! \eta \norm{\phi(c,a)}_{V_{t}^{-1}}$, where $\tau_t \!<\! t$ is the most recent model update time before $t$. To this end, if one uses the determinant trick, one can obtain that
\begin{align*}
    \norm{\phi(c,a)}_{V_{\tau_t}^{-1}} \leq \sqrt{\frac{\det(V_t)}{\det(V_{\tau_t})}} \norm{\phi(c,a)}_{V_{t}^{-1}},
\end{align*}
\emph{if the condition $V_t \succeq V_{\tau_t}$ holds}. Note that this is true in the non-private setting. However, this does not necessarily hold in private settings due to the added noise, which, to the best of our knowledge, is the key analytical gap in the current proof of the main result (Theorem 10) in~\cite{garcelon2021privacy}. As we can see, this issue exists in all three trust models when one needs to use the \emph{noisy} design matrix to determine the update frequency via the determinant trick.

\textbf{Future work.}
One immediate future research direction is to address the above adaptive model update in the private settings. 
We also believe our framework can be generalized to design shuffle private algorithms for reinforcement learning with linear function approximation (e.g., linear mixture Markov decision processes (MDPs)) to 
achieve finer trade-off compared to the local model~\cite{liao2021locally} and the central model~\cite{zhou2022differentially}.

\bibliography{main,Bandit_RL_bib,2018library}

\newcommand{\etalchar}[1]{$^{#1}$}
\begin{thebibliography}{GPPBP20}

\bibitem[ABL03]{abe2003reinforcement}
Naoki Abe, Alan~W Biermann, and Philip~M Long.
\newblock Reinforcement learning with immediate rewards and linear hypotheses.
\newblock {\em Algorithmica}, 37(4):263--293, 2003.

\bibitem[App17]{apple}
Apple.
\newblock Learning with privacy at scale.
\newblock 2017.

\bibitem[AS17]{agarwal2017price}
Naman Agarwal and Karan Singh.
\newblock The price of differential privacy for online learning.
\newblock In {\em International Conference on Machine Learning}, pages 32--40.
  PMLR, 2017.

\bibitem[Aue03]{Auer03confidence}
Peter Auer.
\newblock Using confidence bounds for exploitation-exploration trade-offs.
\newblock {\em Journal of Machine Learning Research}, 3:397--422, March 2003.

\bibitem[AYN14]{arora2014privacy}
Shifali Arora, Jennifer Yttri, and Wendy Nilsen.
\newblock Privacy and security in mobile health (mhealth) research.
\newblock {\em Alcohol research: current reviews}, 36(1):143, 2014.

\bibitem[AYPS11]{abbasi2011improved}
Yasin Abbasi-Yadkori, D{\'a}vid P{\'a}l, and Csaba Szepesv{\'a}ri.
\newblock Improved algorithms for linear stochastic bandits.
\newblock In {\em Advances in Neural Information Processing Systems}, pages
  2312--2320, 2011.

\bibitem[BBGN19a]{balle2019differentially}
Borja Balle, James Bell, Adria Gascon, and Kobbi Nissim.
\newblock Differentially private summation with multi-message shuffling.
\newblock {\em arXiv preprint arXiv:1906.09116}, 2019.

\bibitem[BBGN19b]{balle2019privacy}
Borja Balle, James Bell, Adria Gasc{\'o}n, and Kobbi Nissim.
\newblock The privacy blanket of the shuffle model.
\newblock In {\em Annual International Cryptology Conference}, pages 638--667.
  Springer, 2019.

\bibitem[BEM{\etalchar{+}}17]{bittau2017prochlo}
Andrea Bittau, {\'U}lfar Erlingsson, Petros Maniatis, Ilya Mironov, Ananth
  Raghunathan, David Lie, Mitch Rudominer, Ushasree Kode, Julien Tinnes, and
  Bernhard Seefeld.
\newblock Prochlo: Strong privacy for analytics in the crowd.
\newblock In {\em Proceedings of the 26th Symposium on Operating Systems
  Principles}, pages 441--459, 2017.

\bibitem[Che21]{cheu2021differential}
Albert Cheu.
\newblock Differential privacy in the shuffle model: A survey of separations.
\newblock {\em arXiv preprint arXiv:2107.11839}, 2021.

\bibitem[CJMP21]{cheu2021shuffle}
Albert Cheu, Matthew Joseph, Jieming Mao, and Binghui Peng.
\newblock Shuffle private stochastic convex optimization.
\newblock 2021.

\bibitem[CKS20]{canonne2020discrete}
Cl{\'e}ment~L Canonne, Gautam Kamath, and Thomas Steinke.
\newblock The discrete gaussian for differential privacy.
\newblock In {\em NeurIPS}, 2020.

\bibitem[CLRS11]{Chu2011}
W.~Chu, L.~Li, L.~Reyzin, and R.~E. Schapire.
\newblock Contextual bandits with linear payoff functions.
\newblock In {\em International Conference on Artificial Intelligence and
  Statistics (AISTATS)}, volume~15, pages 208--214, 2011.

\bibitem[CSS10]{chan2010private}
TH~Hubert Chan, Elaine Shi, and Dawn Song.
\newblock Private and continual release of statistics.
\newblock In {\em International Colloquium on Automata, Languages, and
  Programming}, pages 405--417. Springer, 2010.

\bibitem[CSU{\etalchar{+}}19]{cheu2019distributed}
Albert Cheu, Adam Smith, Jonathan Ullman, David Zeber, and Maxim Zhilyaev.
\newblock Distributed differential privacy via shuffling.
\newblock In {\em Annual International Conference on the Theory and
  Applications of Cryptographic Techniques}, pages 375--403. Springer, 2019.

\bibitem[CZ21]{sayakPO}
Sayak~Ray Chowdhury and Xingyu Zhou.
\newblock Differentially private regret minimization in episodic markov
  decision processes.
\newblock {\em arXiv preprint arXiv:2112.10599}, 2021.

\bibitem[CZS21]{chowdhury2021adaptive}
Sayak~Ray Chowdhury, Xingyu Zhou, and Ness Shroff.
\newblock Adaptive control of differentially private linear quadratic systems.
\newblock In {\em 2021 IEEE International Symposium on Information Theory
  (ISIT)}, pages 485--490. IEEE, 2021.

\bibitem[CZZ{\etalchar{+}}20]{chen2020locally}
Xiaoyu Chen, Kai Zheng, Zixin Zhou, Yunchang Yang, Wei Chen, and Liwei Wang.
\newblock (locally) differentially private combinatorial semi-bandits.
\newblock In {\em International Conference on Machine Learning}, pages
  1757--1767. PMLR, 2020.

\bibitem[DNPR10]{dwork2010differential}
Cynthia Dwork, Moni Naor, Toniann Pitassi, and Guy~N Rothblum.
\newblock Differential privacy under continual observation.
\newblock In {\em Proceedings of the forty-second ACM symposium on Theory of
  computing}, pages 715--724, 2010.

\bibitem[DR{\etalchar{+}}14a]{dwork2014DPbook}
Cynthia Dwork, Aaron Roth, et~al.
\newblock The algorithmic foundations of differential privacy.
\newblock {\em Found. Trends Theor. Comput. Sci.}, 9(3-4):211--407, 2014.

\bibitem[DR{\etalchar{+}}14b]{dwork2014algorithmic}
Cynthia Dwork, Aaron Roth, et~al.
\newblock The algorithmic foundations of differential privacy.
\newblock 2014.

\bibitem[Dub21]{dubey2021no}
Abhimanyu Dubey.
\newblock No-regret algorithms for private gaussian process bandit
  optimization.
\newblock In {\em International Conference on Artificial Intelligence and
  Statistics}, pages 2062--2070. PMLR, 2021.

\bibitem[Dwo08]{dwork2008differential}
Cynthia Dwork.
\newblock Differential privacy: A survey of results.
\newblock In {\em International conference on theory and applications of models
  of computation}, pages 1--19. Springer, 2008.

\bibitem[EFM{\etalchar{+}}19]{erlingsson2019amplification}
{\'U}lfar Erlingsson, Vitaly Feldman, Ilya Mironov, Ananth Raghunathan, Kunal
  Talwar, and Abhradeep Thakurta.
\newblock Amplification by shuffling: From local to central differential
  privacy via anonymity.
\newblock In {\em Proceedings of the Thirtieth Annual ACM-SIAM Symposium on
  Discrete Algorithms}, pages 2468--2479. SIAM, 2019.

\bibitem[FMT20]{feldman2020hiding}
Vitaly Feldman, Audra McMillan, and Kunal Talwar.
\newblock Hiding among the clones: A simple and nearly optimal analysis of
  privacy amplification by shuffling.
\newblock {\em arXiv preprint arXiv:2012.12803}, 2020.

\bibitem[GCPP21]{garcelon2021privacy}
Evrard Garcelon, Kamalika Chaudhuri, Vianney Perchet, and Matteo Pirotta.
\newblock Privacy amplification via shuffling for linear contextual bandits.
\newblock {\em arXiv preprint arXiv:2112.06008}, 2021.

\bibitem[GDD{\etalchar{+}}21]{girgis2021shuffled}
Antonious Girgis, Deepesh Data, Suhas Diggavi, Peter Kairouz, and
  Ananda~Theertha Suresh.
\newblock Shuffled model of differential privacy in federated learning.
\newblock In {\em International Conference on Artificial Intelligence and
  Statistics}, pages 2521--2529. PMLR, 2021.

\bibitem[GGK{\etalchar{+}}19]{ghazi2019power}
Badih Ghazi, Noah Golowich, Ravi Kumar, Rasmus Pagh, and Ameya Velingker.
\newblock On the power of multiple anonymous messages.
\newblock {\em arXiv preprint arXiv:1908.11358}, 2019.

\bibitem[GPPBP20]{garcelon2020local}
Evrard Garcelon, Vianney Perchet, Ciara Pike-Burke, and Matteo Pirotta.
\newblock Local differentially private regret minimization in reinforcement
  learning.
\newblock {\em arXiv preprint arXiv:2010.07778}, 2020.

\bibitem[GTS13]{guha2013nearly}
Abhradeep Guha~Thakurta and Adam Smith.
\newblock (nearly) optimal algorithms for private online learning in
  full-information and bandit settings.
\newblock {\em Advances in Neural Information Processing Systems},
  26:2733--2741, 2013.

\bibitem[HHR{\etalchar{+}}16]{hsu2016private}
Justin Hsu, Zhiyi Huang, Aaron Roth, Tim Roughgarden, and Zhiwei~Steven Wu.
\newblock Private matchings and allocations.
\newblock {\em SIAM Journal on Computing}, 45(6):1953--1984, 2016.

\bibitem[HLWZ21]{han2021generalized}
Yuxuan Han, Zhipeng Liang, Yang Wang, and Jiheng Zhang.
\newblock Generalized linear bandits with local differential privacy.
\newblock {\em arXiv preprint arXiv:2106.03365}, 2021.

\bibitem[HZZ{\etalchar{+}}20]{han2020sequential}
Yanjun Han, Zhengqing Zhou, Zhengyuan Zhou, Jose Blanchet, Peter~W Glynn, and
  Yinyu Ye.
\newblock Sequential batch learning in finite-action linear contextual bandits.
\newblock {\em arXiv preprint arXiv:2004.06321}, 2020.

\bibitem[KLS21]{kairouz2021distributed}
Peter Kairouz, Ziyu Liu, and Thomas Steinke.
\newblock The distributed discrete gaussian mechanism for federated learning
  with secure aggregation.
\newblock In {\em NeurIPS}, 2021.

\bibitem[KPRU14]{kearns2014mechanism}
Michael Kearns, Mallesh Pai, Aaron Roth, and Jonathan Ullman.
\newblock Mechanism design in large games: Incentives and privacy.
\newblock In {\em Proceedings of the 5th conference on Innovations in
  theoretical computer science}, pages 403--410, 2014.

\bibitem[LCLS10]{li2010contextual}
Lihong Li, Wei Chu, John Langford, and Robert~E Schapire.
\newblock A contextual-bandit approach to personalized news article
  recommendation.
\newblock In {\em Proceedings of the 19th international conference on World
  wide web}, pages 661--670, 2010.

\bibitem[LHG21]{liao2021locally}
Chonghua Liao, Jiafan He, and Quanquan Gu.
\newblock Locally differentially private reinforcement learning for linear
  mixture markov decision processes.
\newblock {\em arXiv preprint arXiv:2110.10133}, 2021.

\bibitem[LR21]{Lowy2021}
Andrew Lowy and Meisam Razaviyayn.
\newblock Private federated learning without a trusted server: Optimal
  algorithms for convex losses.
\newblock {\em arXiv preprint arXiv:2106.09779}, 2021.

\bibitem[Mir12]{mironov2012significance}
Ilya Mironov.
\newblock On significance of the least significant bits for differential
  privacy.
\newblock In {\em Proceedings of the 2012 ACM conference on Computer and
  communications security}, pages 650--661, 2012.

\bibitem[MT15]{mishra2015nearly}
Nikita Mishra and Abhradeep Thakurta.
\newblock (nearly) optimal differentially private stochastic multi-arm bandits.
\newblock In {\em Proceedings of the Thirty-First Conference on Uncertainty in
  Artificial Intelligence}, pages 592--601, 2015.

\bibitem[RZK20]{ren2020batched}
Zhimei Ren, Zhengyuan Zhou, and Jayant~R Kalagnanam.
\newblock Batched learning in generalized linear contextual bandits with
  general decision sets.
\newblock {\em IEEE Control Systems Letters}, 2020.

\bibitem[RZLS20]{ren2020multi}
Wenbo Ren, Xingyu Zhou, Jia Liu, and Ness~B Shroff.
\newblock Multi-armed bandits with local differential privacy.
\newblock {\em arXiv preprint arXiv:2007.03121}, 2020.

\bibitem[SBF17]{schwartz2017customer}
Eric~M Schwartz, Eric~T Bradlow, and Peter~S Fader.
\newblock Customer acquisition via display advertising using multi-armed bandit
  experiments.
\newblock {\em Marketing Science}, 36(4):500--522, 2017.

\bibitem[SS18]{shariff2018differentially}
Roshan Shariff and Or~Sheffet.
\newblock Differentially private contextual linear bandits.
\newblock {\em Advances in Neural Information Processing Systems},
  31:4296--4306, 2018.

\bibitem[SS19]{sajed2019optimal}
Touqir Sajed and Or~Sheffet.
\newblock An optimal private stochastic-mab algorithm based on optimal private
  stopping rule.
\newblock In {\em International Conference on Machine Learning}, pages
  5579--5588. PMLR, 2019.

\bibitem[TD17]{tossou2017achieving}
Aristide Tossou and Christos Dimitrakakis.
\newblock Achieving privacy in the adversarial multi-armed bandit.
\newblock In {\em Proceedings of the AAAI Conference on Artificial
  Intelligence}, volume~31, 2017.

\bibitem[TKMS21]{tenenbaum2021differentially}
Jay Tenenbaum, Haim Kaplan, Yishay Mansour, and Uri Stemmer.
\newblock Differentially private multi-armed bandits in the shuffle model.
\newblock In A.~Beygelzimer, Y.~Dauphin, P.~Liang, and J.~Wortman Vaughan,
  editors, {\em Advances in Neural Information Processing Systems}, 2021.

\bibitem[TM17]{tewari2017ads}
Ambuj Tewari and Susan~A Murphy.
\newblock From ads to interventions: Contextual bandits in mobile health.
\newblock In {\em Mobile Health}, pages 495--517. Springer, 2017.

\bibitem[VBKW20]{vietri2020private}
Giuseppe Vietri, Borja Balle, Akshay Krishnamurthy, and Steven Wu.
\newblock Private reinforcement learning with pac and regret guarantees.
\newblock In {\em International Conference on Machine Learning}, pages
  9754--9764. PMLR, 2020.

\bibitem[Ver18]{vershynin2018high}
Roman Vershynin.
\newblock {\em High-dimensional probability: An introduction with applications
  in data science}, volume~47.
\newblock Cambridge university press, 2018.

\bibitem[VMDK20]{vaswani2020old}
Sharan Vaswani, Abbas Mehrabian, Audrey Durand, and Branislav Kveton.
\newblock Old dog learns new tricks: Randomized ucb for bandit problems.
\newblock In {\em International Conference on Artificial Intelligence and
  Statistics}, pages 1988--1998. PMLR, 2020.

\bibitem[WZG21]{wang2021provably}
Tianhao Wang, Dongruo Zhou, and Quanquan Gu.
\newblock Provably efficient reinforcement learning with linear function
  approximation under adaptivity constraints.
\newblock {\em arXiv preprint arXiv:2101.02195}, 2021.

\bibitem[XJ14]{xin2014controlling}
Yu~Xin and Tommi Jaakkola.
\newblock Controlling privacy in recommender systems.
\newblock Neural Information Processing Systems, 2014.

\bibitem[ZCH{\etalchar{+}}20]{zheng2020locally}
Kai Zheng, Tianle Cai, Weiran Huang, Zhenguo Li, and Liwei Wang.
\newblock Locally differentially private (contextual) bandits learning.
\newblock In {\em NeurIPS}, 2020.

\bibitem[Zho22]{zhou2022differentially}
Xingyu Zhou.
\newblock Differentially private reinforcement learning with linear function
  approximation.
\newblock {\em arXiv preprint arXiv:2201.07052}, 2022.

\bibitem[ZT20]{zhou2020local}
Xingyu Zhou and Jian Tan.
\newblock Local differential privacy for bayesian optimization.
\newblock {\em arXiv preprint arXiv:2010.06709}, 2020.

\end{thebibliography}
\bibliographystyle{alpha}

%%%%%%%%%%%%%%%%%%%%%%%%%%%%%%%%%%%%%%%%%%%%%%%%%%%%%%%%%%%%%%%%%%%%%%%%%%%%%%%
%%%%%%%%%%%%%%%%%%%%%%%%%%%%%%%%%%%%%%%%%%%%%%%%%%%%%%%%%%%%%%%%%%%%%%%%%%%%%%%
% APPENDIX
%%%%%%%%%%%%%%%%%%%%%%%%%%%%%%%%%%%%%%%%%%%%%%%%%%%%%%%%%%%%%%%%%%%%%%%%%%%%%%%
%%%%%%%%%%%%%%%%%%%%%%%%%%%%%%%%%%%%%%%%%%%%%%%%%%%%%%%%%%%%%%%%%%%%%%%%%%%%%%%
% \newpage
% \appendix
% %\onecolumn
% % \section{You \emph{can} have an appendix here.}
% \input{appendix}

\begin{appendix}

\section{A Unified Regret Analysis Under Differential Privacy}\label{app:unified_regret}
% \xingyu{in progress...it turns out the method in RL paper~\cite{wang2021provably} for bounding sum variance cannot be directly used,  since it requires a stronger assumption in the proof, which is not necessarily satisfied in the private case (i.e., the first two inequality in the proof of Lemma 4.4. I will double-check). But, we can turn to use the other one~\cite{ren2020batched}. However, when I double-check the first step of iterative eigenvalue decomposition, I think it also needs care due to the additional noise term. 

% Another problem is that  the other one in the non-private needs care for the regularizer, see Remark 5.1~\cite{ren2020batched}. As a result, it could lead to sub-optimal in terms of $d$ in our case. Thus, if the method in RL paper is applicable, I prefer it over the one in~\cite{ren2020batched}}

In this section, we will formally state Lemma~\ref{lem:subG-main}, i.e., the generic regret of Algorithm~\ref{alg:BOFUL} under sub-Gaussian private noise and then present its proof. 

Let's first recall the following notations. For each batch $m \in [M]$, let $N_m := \widetilde{V}_m - \sum_{t = t_{m-1}+1}^{t_m}\phi(c_t,a_t)\phi(c_t,a_t)^{\top}$ denote the additional noise injected into the non-private Gram-matrix  and similarly let $n_m:= \widetilde{u}_m - \sum_{t = t_{m-1}+1}^{t_m} \phi(c_t, a_t)y_t$ denote the additional noise injected into the non-private feature-reward vector. 
Then, we let $H_{m}:= \lambda I_d + \sum_{i=1}^m N_i$ to denote the \emph{total} noise in the first $m$ batches plus the regularizer, and similarly let $h_{m} := \sum_{i=1}^m n_i$.

% We first present a general regret bound for Algorithm~\ref{alg:BOFUL} under a assumption on the injected noise. To this end, we first introduce the following necessary notations. 

\begin{assumption}[Regularity]
\label{ass:reg}
For any $\alpha \in (0,1]$,  $H_m$ is positive definite and there exist constants $\lambda_{\max}$, $\lambda_{\min}$ and $\nu$ depending on $\alpha$, such that with probability at least $1-\alpha$, for all $m \in [M]$
\begin{align*}
    \norm{H_{m}} \le \lambda_{\max},\quad \norm{H_{m}^{-1}} \le 1/\lambda_{\min}, \quad \norm{h_m}_{H_m^{-1}} \le \nu.
\end{align*}
\end{assumption}

With the above regularity assumption and the boundedness in Assumption~\ref{ass:bounded}, we fist establish the following general regret bound of Algorithm~\ref{alg:BOFUL}, which can be viewed as a direct generalization of the results in~\cite{shariff2018differentially} to the batched case. 

\begin{lemma}
\label{thm:general}
Let Assumptions~\ref{ass:reg} and~\ref{ass:bounded} hold. Fix any $\alpha \in (0,1]$, with probability at least $1-\alpha$, the regret of Algorithm~\ref{alg:BOFUL} satisfies 
\begin{align*}
    \text{Reg}(T) \le \frac{dB}{\log 2 }\log\left(1 + \frac{T}{d\lambda_{\min}}\right) + 8\beta_M \sqrt{dT\log\left(1+\frac{T}{d\lambda_{\min}}\right)},
\end{align*}
where 
\begin{align*}
    \beta_M:=\sqrt{2\log\left(\frac{2}{\alpha}\right) + d\log\left(1+\frac{T}{d\lambda_{\min}}\right)} + \sqrt{\lambda_{\max}} + \nu.
\end{align*}
\end{lemma}
% \xingyu{The above general theorem even improves the non-private one a little bit.}

In fact, Lemma~\ref{lem:subG-main} in the main paper is a simple application of Lemma~\ref{thm:general} by considering the following assumption. 
% We consider the following specific case of sub-Gaussian private noise. 
\begin{assumption}[sub-Gaussian private noise]
\label{ass:subG}
There exist constants $\widetilde{\sigma}_1$ and $\widetilde{\sigma}_2$ such that for all $m \in [M]$:  (i) $\sum_{m'=1}^m n_{m'}$ is a random vector whose entries are independent, mean zero, sub-Gaussian with variance at most $\widetilde{\sigma}_1^2$, and
(ii) $\sum_{m'=1}^m N_{m'}$ is a random symmetric matrix whose entries on and above the diagonal are independent sub-Gaussian random variables with variance at most $\widetilde{\sigma}_2^2$. Let $\sigma^2 \!=\! \max\lbrace \widetilde \sigma_1^2,\widetilde \sigma_2^2\rbrace$.
\end{assumption}

% With the above assumption, as a corollary of Theorem~\ref{thm:general}, we have the following regret bound under sug-Gaussian private noise.
Now, we are well-prepared to formally state Lemma~\ref{lem:subG-main} in the main paper.

\begin{lemma}[Formal statement of Lemma~\ref{lem:subG-main}]
\label{cor:subG}
Let Assumptions~\ref{ass:subG} and~\ref{ass:bounded} hold. Fix time horizon $T \in \Nat$, batch size $B \in [T]$, confidence level $\alpha \in (0,1]$. Set $\lambda = \Theta(\max\{1, \sigma (\sqrt{d}+ \sqrt{\log(T/(B\alpha))}\})$ and $\beta_m = \sqrt{2\log\left(\frac{2}{\alpha}\right) + d\log\left(1+\frac{T}{d\lambda}\right)} + \sqrt{\lambda} $. Then, Algorithm~\ref{alg:BOFUL} achieves regret 
\begin{align*}
    Reg(T) = O\left({dB}\log T + d\sqrt{T}\log(T/\alpha)\right) + O\left(\sqrt{\sigma T}d^{3/4}\log T \log (T/\alpha)\right)
\end{align*}
with probability at least $1-\alpha$.
\end{lemma}

\begin{remark}
The above lemma also presents a regret bound for non-private batched LCB when $\sigma = 0$. Note that in this case, our regret bound is achieved with a \emph{dimension-independent} regularizer, in contrast to the necessary condition on $\lambda = \tilde{\Theta}(d)$ as required in~\cite{ren2020batched} to attain the optimal regret.
\end{remark}

% \begin{corollary}
% \label{cor:subG}
% Let Assumptions~\ref{ass:subG} and~\ref{ass:bounded} hold, and $\lambda = O(\sigma (\sqrt{d}+ \sqrt{\log(M/\alpha)})$. For any $\alpha \in (0,1]$, with probability at least $1-2\alpha$, the regret of Algorithm~\ref{alg:BOFUL} satisfies 
% \begin{align*}
%     R(T) = O\left({dB}\log T + d\sqrt{T}\log(T/\alpha)\right) + O\left(\sqrt{\sigma T}d^{3/4}\log T \log (T/\alpha)\right),
% \end{align*}
% where $\sigma = \max\{\widetilde{\sigma}_1, \widetilde{\sigma}_2\}$.
% \end{corollary}

\subsection{Proofs}
In this section, we present proofs for Lemma~\ref{thm:general} and Lemma~\ref{cor:subG} above, respectively. 
\begin{proof}[Proof of Lemma~\ref{thm:general}]
Let $\cE$ be the event given in Assumption~\ref{ass:reg}, which holds with probability at least $1-\alpha$ under Assumption~\ref{ass:reg}. In the following, we condition on the event $\cE$.
We first show that $\hat{\theta}_m$  concentrates around the true parameter $\theta^*$ with a properly chosen confidence radius $\beta_m$ for all $m \in [M]$. To this end, note that 
\begin{align*}
    \hat{\theta}_m &= V_m^{-1} u_m \\
    &=\left(\sum_{t=1}^{t_m} \phi(c_t,a_t) \phi(c_t,a_t)^{\top} + \lambda I_d + \sum_{i=1}^m N_m \right)^{-1}\left( \sum_{t=1}^{t_m} \phi(c_t, a_t)y_t + \sum_{i=1}^m n_m \right)\\
    &=\left(\sum_{t=1}^{t_m} \phi(c_t,a_t) \phi(c_t,a_t)^{\top} + H_m \right)^{-1}\left( \sum_{t=1}^{t_m} \phi(c_t, a_t)y_t + h_m \right).
\end{align*}
By the linear reward function $y_t = \inner{\phi(c_t,a_t)}{\theta^*} + \eta_t$ and elementary algebra, we have 
\begin{align*}
    \theta^* - \hat{\theta}_m = V_m^{-1}\left(H_m \theta^* - \sum_{t=1}^{t_m}\phi(c_t,a_t)\eta_t - h_m\right).
\end{align*}
Thus, multiplying both sides by $V_m^{1/2}$, yields
\begin{align*}
    \norm{\theta^* - \hat{\theta}_m}_{V_m} &\le \norm{\sum_{t=1}^{t_m}\phi(c_t,a_t)\eta_t}_{V_m^{-1}} + \norm{H_m\theta^*}_{V_m^{-1}} + \norm{h_m}_{V_m^{-1}}\\
    & \lep{a} \norm{\sum_{t=1}^{t_m}\phi(c_t,a_t)\eta_t}_{(G_m + \lambda_{\min}I)^{-1}} + \norm{\theta^*}_{H_m} + \norm{h_m}_{H_m^{-1}},
\end{align*}
where (a) holds by $V_m \succeq H_m$ and $V_m \succeq G_m + \lambda_{\min} I$ with $G_m:= \sum_{t=1}^{t_m}\phi(c_t,a_t)\phi(c_t,a_t)^{\top}$ under event $\cE$. Further, by the boundedness condition of $\theta^*$ and event $\cE$, $\norm{\theta^*}_{H_m} \le \sqrt{\lambda_{\max}}$ and $\norm{h_m}_{H_m^{-1}} \le \nu$. For the remaining first term, we can use self-normalized inequality (cf. Theorem 1 in~\cite{abbasi2011improved}) with the filtration $\cF_t = \sigma(c_1,a_1,y_1,\ldots,c_t,a_t,y_t, c_{t+1},a_{t+1})$.
In particular, we have with probability at least $1-\alpha$, for all $m\in [M]$
\begin{align}
\label{eq:self-nor}
    \norm{\sum_{t=1}^{t_m}\phi(c_t,a_t)\eta_t}_{(G_m + \lambda_{\min}I)^{-1}} \le \sqrt{2\log\left(\frac{1}{\alpha}\right) + \log\left(\frac{\det(G_m + \lambda_{\min}I) }{\det(\lambda_{\min} I)}\right)}.
\end{align}
Now, using the trace-determinant lemma (cf. Lemma 10 in~\cite{abbasi2011improved}) and the boundedness condition on $\norm{\phi(c,a)}$, we have 
\begin{align*}
    \det(G_m + \lambda_{\min}I) \le \left(\lambda_{\min} + \frac{t_m }{d}\right)^d.
\end{align*}
Putting everything together, we have with probability at least $1-2\alpha$, for all $m \in [M]$, $\norm{\theta^* - \hat{\theta}_m}_{V_m} \le \beta_m$, where 
\begin{align*}
    \beta_m:=\sqrt{2\log\left(\frac{1}{\alpha}\right) + d\log\left(1+\frac{t_m}{d\lambda_{\min}}\right)} + \sqrt{\lambda_{\max}} + \nu.
\end{align*}
With the above concentration result and our OFUL-type algorithm, the regret can be upper bounded as follows.
\begin{align}
\label{eq:reg_temp}
    \cR(T) &= \sum_{m=1}^M \left[2\beta_{m-1} \sum_{t=t_{m-1}+1}^{t_m} \left(\norm{\phi(c_t,a_t)}_{V_{m-1}^{-1}}\right) \right]\nonumber\\
    &\le \sum_{m=1}^M \left[2\beta_{M} \sum_{t=t_{m-1}+1}^{t_m} \left(\norm{\phi(c_t,a_t)}_{(G_{m-1} + \lambda_{\min} I)^{-1}}\right) \right]
\end{align}
At this moment, we note that the standard elliptical potential lemma (cf. Lemma 11 in~\cite{abbasi2011improved}) cannot be applied to our batch setting due to the delay of $G_m$.

To handle this, inspired by~\cite{wang2021provably}, we let $\hat{V}_k:= \sum_{t=1}^{k} \phi(c_t,a_t) \phi(c_t,a_t)^{\top} + \lambda_{\min} I_d$, that is, a (virtual) design matrix at the end of time $k$. Hence, we have $G_{m-1} + \lambda_{\min} I_d = \hat{V}_{t_{(m-1)}}$. Moreover, for any $t_{m-1} < t \le t_m$,  let $m_t = t_{m-1}$, that is, mapping $t$ to the starting time of the batch that includes $t$. Finally, let $\Gamma_i(\cdot,\cdot) := \beta_M\cdot \norm{\phi(\cdot,\cdot)}_{\hat{V}_{i}^{-1}}$.

With above notations, the bound in~\eqref{eq:reg_temp} can be rewritten as follows.
\begin{align*}
    R(T) &\le \sum_{m=1}^M \left[2\beta_{M} \sum_{t=t_{m-1}+1}^{t_m} \left(\norm{\phi(c_t,a_t)}_{(G_{m-1} + \lambda_{\min})^{-1}}\right) \right]\\
    &=\sum_{t=1}^T 2\Gamma_{m_t}(c_t, a_t)
\end{align*}
% where the last inequality follows from the boundedness of reward. 
In the sequential case (i.e., $B=1$), we always have $m_t = t-1$. Thus, the key is to bound the difference between $\Gamma_{m_t}(c_t,a_t)$ and $\Gamma_{t-1}(c_t,a_t)$. To this end, we have the following claim, which will be proved at the end.
\begin{claim}
\label{clm:dif}
Define the set $\Psi$ as follows
\begin{align*}
    \Psi = \{t \in [T]: \Gamma_{m_t}(c_t,a_t)/\Gamma_{t-1}(c_t,a_t) > 2\}.
\end{align*}
Then, we have
\begin{align*}
    |\Psi| \le \frac{dT}{2M\log 2 }\log\left(1 + \frac{T}{d\lambda_{\min}}\right).
\end{align*}
\end{claim}

According to Claim~\ref{clm:dif}, we can decompose regret as follows. 
\begin{align*}
    R(T) &\lep{a} \sum_{t=1}^T \min\{2\Gamma_{m_t}(c_t, a_t),1\}\\
    &= \sum_{t \in \Psi} \min\{2\Gamma_{m_t}(c_t, a_t),1\} + \sum_{t \notin \Psi} \min\{2\Gamma_{m_t}(c_t, a_t),1\}\\
    &\lep{b} |\Psi| + \sum_{t \notin \Psi} \min\{4\Gamma_{t-1}(c_t,a_t) ,1\}\\
    &\lep{c} |\Psi| + \sum_{t=1}^T 4\beta_M \min\{\norm{\phi(c_t,a_t)}_{\hat{V}_{t-1}^{-1}} ,1 \}\\
    &\lep{d} \frac{dT}{2M\log 2 }\log\left(1 + \frac{T}{d\lambda_{\min}}\right) + 8\beta_M \sqrt{dT\log\left(1+\frac{T}{d\lambda_{\min}}\right)}
\end{align*}
where (a) holds by the boundedness of reward; (b) holds by definition of $\Psi$; (c) holds by the fact that $\beta_M \ge 1$; (d) follows from Claim~\ref{clm:dif} and standard argument for linear bandit, i.e., Cauchy-Schwartz and standard elliptical potential lemma (cf. Lemma 11 in~\cite{abbasi2011improved}). Hence, we have finished the proof of Lemma~\ref{thm:general}.

Finally, we give the proof of Claim~\ref{clm:dif}.

For any  $t \in \Psi$, suppose $t_{m-1} < t \le t_m$ for some $m$. Then, we have $m_t = t_{(m-1)}$ and 
\begin{align*}
    \log\det(\hat{V}_{t_m}) - \log\det(\hat{V}_{t_{(m-1)}}) \gep{a} \log\det(\hat{V}_{t-1}) - \log\det(\hat{V}_{m_t}) \gep{b} 2\log(\Gamma_{m_t}(c_t, a_t) / \Gamma_{t-1}(c_t,a_t)) > 2\log 2,
\end{align*}
where (a) holds by the fact $\hat{V}_{t_m} \succeq \hat{V}_{t-1}$; (b) holds by Lemma 12 in~\cite{abbasi2011improved}, that is, for two positive definite matrices  $A, B\in \mathbb{R}^{d \times d}$ satisfying $A \succeq B$, then for any $x \in \mathbb{R}^d$, $\norm{x}_{A} \le \norm{x}_B \cdot \sqrt{\det(A)/\det(B)}$. Note that here we also use $\det(A) = 1/\det(A^{-1})$ for any matrix; 

Therefore, if we let $\hat{\Psi}:=\{m \in [M]:  \log\det(\hat{V}_{t_m }) - \log\det(\hat{V}_{t_{(m-1)}}) > 2\log 2\}$, then we have $|\Psi| \le (T/M) |\hat{\Psi}|$. Thus, we only need to bound $|\hat{\Psi}|$. Note that for each $m$, $\log\det(\hat{V}_{t_m }) - \log\det(\hat{V}_{t_{(m-1)}}) \ge 0$, and hence 
\begin{align*}
    2\log 2 \cdot |\hat{\Psi}| \le \sum_{m \in \hat{\Psi}}\log\det(\hat{V}_{t_m }) - \log\det(\hat{V}_{t_{(m-1)}}) &\le \sum_{m=1}^M \log\det(\hat{V}_{t_m}) - \log\det(\hat{V}_{t_{(m-1)}})\\ &= \log\left(\frac{\det(G_M + \lambda_{\min}I) }{\det(\lambda_{\min} I)}\right)
\end{align*}
Finally, using the same analysis as in~\eqref{eq:self-nor}, yields 
\begin{align*}
    |\hat{\Psi}| \le \frac{d}{2\log 2} \log\left(1 + \frac{T}{d\lambda_{\min}}\right),
\end{align*}
which directly implies the result of Claim~\ref{clm:dif}.
% To address this issue,  we need to carefully bound the difference between the delayed bonus term $\norm{\phi(c_t,a_t)}_{V_{m-1}^{-1}}$ and the standard one $\norm{\phi(c_t,a_t)}_{V_{t-1}^{-1}}$. Inspired by [],
% for any $t \in [T]$, let $m_t$ be the batch $m$ such that $t_{m-1} < t \le t_{m}$ and also for notation simplicity, we let $\Gamma_i(\cdot,\cdot) := \norm{\phi(\cdot,\cdot)}_{V_{i-1}^{-1}}$. Therefore, we have 
% \begin{align}
%     \cR(T) \le \sum_{t=1}^T 2\beta_M \min\{\Gamma_{m_t}(c_t,a_t),1\}.
% \end{align}
\end{proof}

\begin{proof}[Proof of Lemma~\ref{cor:subG}]
To prove the result, thanks to Lemma~\ref{thm:general}, we only need to determine the three constants $\lambda_{\max}, \lambda_{\min}$ and $\nu$ under the sub-Gaussian private noise assumption in Assumption~\ref{ass:subG}. To this end, we resort to concentration bounds for sub-Gaussian random vector and random matrix. 

To start with, under (i) in Assumption~\ref{ass:subG}, by the concentration bound for the norm of a vector containing sub-Gaussian entries (cf. Theorem 3.1.1 in~\cite{vershynin2018high}) and a union bound over $m$, we have for all $m \in [M]$ and any $\alpha \in (0,1]$, with probability at least $1-\alpha/2$, for some absolute constant $c_1$,
\begin{align*}
    \norm{\sum_{i=1}^m n_i} = \norm{h_m} \le \Sigma_n:= c_1\cdot \widetilde{\sigma}_1 \cdot(\sqrt{d} + \sqrt{\log(M/\alpha)}).
\end{align*}

By (ii) in Assumption~\ref{ass:subG}, the concentration bound for the norm of a sub-Gaussian symmetric random matrix (cf. Corollary 4.4.8~\cite{vershynin2018high}) and a union bound over $m$, we have for all $m \in [M]$ and any $\alpha \in (0,1]$, with probability at least $1-\alpha/2$, 
\begin{align*}
    \norm{\sum_{i=1}^m N_i} \le \Sigma_N:= c_2\cdot \widetilde{\sigma}_2 \cdot(\sqrt{d} + \sqrt{\log(M/\alpha)})
\end{align*}
for some absolute constant $c_2$. Thus, if we choose $\lambda = 2 \Sigma_N$, we have $\norm{H_m} = \norm{\lambda I_d + \sum_{i=1}^m N_i} \le 3\Sigma_N$, i.e., $\lambda_{\max} = 3\Sigma_N$, and $\lambda_{\min} = \Sigma_N$. Finally, to determine $\nu$, we note that 
\begin{align*}
    \norm{h_m}_{H_m^{-1}} \le \frac{1}{\sqrt{\lambda_{\min}}} \norm{h_m} \le c\cdot \left(\sigma \cdot (\sqrt{d} + \sqrt{\log(M/\alpha)})\right)^{1/2} := \nu,
\end{align*}
where $\sigma= \max\{\widetilde{\sigma}_1,\widetilde{\sigma}_2\}$. The final regret bound is obtained by plugging the three values into the result given by Lemma~\ref{thm:general}.
\end{proof}
%%%%%%%%%%%%%%%%%%%%%%%%%%%%%%%%%%%%%%%%%%%%%%%%%%%%%%%%%%%%%%%%%%%%%%%%%%%%%%%
%%%%%%%%%%%%%%%%%%%%%%%%%%%%%%%%%%%%%%%%%%%%%%%%%%%%%%%%%%%%%%%%%%%%%%%%%%%%%%%

\section{Analysis of LDP Amplification Protocol}\label{app:amp}

\subsection{Pseudocode of $\cP_{\text{Amp}}$}
The shuffle protocol is given by  $\mathcal{P}_{\text{Amp}} \!=\! (\cR_{\text{Amp}}, \cS_{\text{Amp}}, \cA_{\text{Amp}})$, in which $\cR_{\text{Amp}}$ is presented in Algorithm~\ref{alg:Ramp}, $\cS_{\text{Amp}}$ is presented in Algorithm~\ref{alg:Samp}, and $\cA_{\text{Amp}}$ is presented in Algorithm~\ref{alg:Aamp}.
\begin{algorithm}[tb]
  \caption{Local Randomizer $\cR_{\text{Amp}}$ }
  \label{alg:Ramp}
\begin{algorithmic}[1]
  \STATE {\bfseries Parameters:} $\sigma_1, \sigma_2, d$
%   \STATE {\bfseries Initialize:} 
  \FUNCTION{$R_1(\phi(c,a)y)$}
    \STATE Sample fresh noise $n \sim \cN(0, \sigma_2^2 I_{d \times d})$
    \STATE $M_{1} = \phi(c,a)y + n$
     \STATE {\bfseries return} $M_{t,1}$
  \ENDFUNCTION
  \FUNCTION{$R_2(\phi(c,a)\phi(c,a)^\top)$}
    \STATE Sample fresh noise $N_{(i,j)} \sim \cN(0, \sigma_1^2), \forall i \le j \le d$ and let $N_{(j,i)} = N_{(i,j)}$
    \STATE $M_2 = \phi(c,a)\phi(c,a)^\top + N$
    \STATE {\bfseries return} $M_2$
  \ENDFUNCTION
  
\end{algorithmic}
\end{algorithm}

\begin{algorithm}[tb]
  \caption{Shuffler $\cS_{\text{Amp}}$ }
  \label{alg:Samp}
\begin{algorithmic}[1]
    \STATE {\bfseries Input:} $\{  M_{\tau, 1} \}_{\tau \in \cB}$ and $\{  M_{\tau, 2} \}_{\tau \in \cB}$, in which $\cB$ is a batch and  $M_{\tau, 1} \in \mathbb{R}^d$, $M_{\tau,2} \in \mathbb{R}^{d\times d}$ come from user $\tau$
  \FUNCTION{$S_1( \{  M_{\tau, 1} \}_{\tau \in \cB})$}
    \STATE Generate a uniform permutation $\pi$ of indexes in $\cB$
    \STATE Set $Y_1 = (M_{\pi(1), 1}, \ldots, M_{\pi(B), 1})$
    \STATE {\bfseries return} $Y_1$
  \ENDFUNCTION
  \FUNCTION{$S_2( \{  M_{\tau, 2} \}_{\tau \in \cB})$}
    \STATE Generate a uniform permutation $\pi$ of indexes in $\cB$
    \STATE Set $Y_2 = (M_{\pi(1), 2}, \ldots, M_{\pi(B), 2})$
    \STATE {\bfseries return} $Y_2$
  \ENDFUNCTION
\end{algorithmic}
\end{algorithm}

\begin{algorithm}[!tb]
  \caption{Analyzer $\cA_{\text{Amp}}$ }
  \label{alg:Aamp}
\begin{algorithmic}[1]
    \STATE {\bfseries Input:} Shuffled outputs $ Y_1 = (M_{\pi(1), 1}, \ldots, M_{\pi(B), 1})$ and $Y_2 = (M_{\pi(1), 2}, \ldots, M_{\pi(B), 2})$
  \FUNCTION{$A_1( Y_1 )$}
    \STATE {\bfseries return} $\sum_{i=1}^B M_{\pi(i),1}$
  \ENDFUNCTION
  \FUNCTION{$A_1( Y_2  )$}
    \STATE {\bfseries return} $\sum_{i=1}^B M_{\pi(i),2}$
  \ENDFUNCTION
\end{algorithmic}
\end{algorithm}

\subsection{Main Results}
\begin{theorem}[Restatement of Theorem~\ref{thm:amp-main}]
\label{thm:amp}
Fix time horizon $T \in \Nat$, batch size $B \in [T]$, confidence level $\alpha \in (0,1]$, , and privacy budgets $\epsilon \in (0,\sqrt{\frac{{\log(2/\delta)}}{{B}}}]$, $\delta \in (0,1]$.  Then, Algorithm~\ref{alg:BOFUL} instantiated with shuffle protocol $\cP_{\text{Amp}}$ with noise levels $\sigma_1\!=\!\sigma_2\!=\! \frac{4\sqrt{2\log(2.5B/\delta) \log(2/\delta)}}{\epsilon\sqrt{B}}$, and regularizer $\lambda = \Theta(\sqrt{T}\sigma_1 (\sqrt{d}+ \sqrt{\log(T/B\alpha)})$, enjoys the regret
\begin{align*}
     \text{Reg}(T) \!=\! O\!\!\left(\!dB\log T \!+\! \frac{\log^{1/4}(B/\delta)\log^{1/4}(2/\delta)}{\epsilon^{1/2}B^{1/4}} d^{3/4} T^{3/4}\log T\log(T/\alpha)\!\!\right),
\end{align*}
with probability at least $1-\alpha$. Moreover, it satisfies $O(\epsilon,\delta)$-shuffle differential privacy (SDP). 
% Let Assumption~\ref{ass:bounded} hold. For any $\epsilon \in [0,1/\sqrt{B}]$, $\delta \in (0,1]$ and $B \in [T]$, let $\sigma_1 = \sigma_2 = \frac{4\sqrt{2\log(2.5B/\delta)}}{\epsilon\sqrt{B}}$. Then, $\cP_{\text{Amp}}$ is $O(\epsilon,\delta)$-SDP. Meanwhile, the corresponding regret satisfies 
% \begin{align*}
%      R(T) = \tilde{O}\left(dB + \frac{(\log(B/\delta))^{1/4}}{\sqrt{\epsilon}} d^{3/4} B^{-1/4} T^{3/4}\right)
% \end{align*}
\end{theorem}
% \begin{corollary}
% Under the same assumption in Theorem~\ref{thm:amp}. Let $B=1$, $\cP_{\text{Amp}}$ reduces to a standard $O(\epsilon,\delta)$-LDP protocol in the sequential setting with regret $\tilde{O}(\frac{\log(T/\delta)^{1/4}}{\sqrt{\epsilon}} d^{3/4}  T^{3/4})$
% \end{corollary}

\begin{corollary}[Utility-targeted]
\label{cor:amp_util}
Under the same assumption in Theorem~\ref{thm:amp} and Algorithm~\ref{alg:BOFUL} is instantiated with $\cP_{\text{Amp}}$. Let $B = O(d^{-1/5}\epsilon^{-2/5}T^{3/5}\log^{1/5}(T/\delta)\log^{1/5}(2/\delta)$, Algorithm~\ref{alg:BOFUL} achieves $O(\epsilon,\delta)$-SDP with regret 
\begin{align*}
    Reg(T) = \widetilde{O}\left(d^{4/5} T^{3/5}\epsilon^{-2/5}\log^{1/5}(T/\delta) \log^{1/5}(2/\delta)\right).
\end{align*}
Simultaneously, Algorithm~\ref{alg:BOFUL} also achieves ${O}(\epsilon,\delta)$-JDP and ${O}(\epsilon_0,\delta_0)$-LDP where 
\begin{align*}
    \epsilon_0 = O\left(\epsilon^{4/5} T^{3/10}d^{-1/10}\log^{1/10}(T/\delta) \log^{-2/5}(2/\delta)\right),\quad \delta_0 = O\left(\delta d^{1/5}T^{-3/5}\epsilon^{2/5}\log^{-1/5}(T/\delta)\log^{-1/5}(2/\delta)\right).
\end{align*}
\end{corollary}

\begin{corollary}[Privacy-targeted]
\label{cor:amp_priv}
Let Assumption~\ref{ass:bounded} hold and Algorithm~\ref{alg:BOFUL} is instantiated with $\cP_{\text{Amp}}$. For any $\epsilon_0 \in [0,1]$ and $\delta_0 \in (0,1]$, let $\sigma_1 = \sigma_2 = \frac{4\sqrt{2\log(2.5/\delta_0)}}{\epsilon_0}$. Then, for all $B \in [T]$, Algorithm~\ref{alg:BOFUL} is $(\epsilon_0,\delta_0)$-LDP. Further suppose $B = O(d^{-1/4}T^{3/4}\epsilon_0^{-1/2}\log^{1/4}(1/\delta_0))$, then 
 Algorithm~\ref{alg:BOFUL} achieves regret  
\begin{align*}
    Reg(T) = \tilde{O}\left(d^{3/4}T^{3/4}\epsilon_0^{-1/2}\log^{1/4}(1/\delta_0)\right).
\end{align*}
Simultaneously, Algorithm~\ref{alg:BOFUL} achieves $O(\epsilon,\delta)$-SDP and $O(\epsilon,\delta)$-JDP where 
\begin{align*}
    \epsilon = O\left(\epsilon_0^{5/4}T^{-3/8} d^{1/8} \log^{3/8}(1/\delta_0)\right), \quad \delta = O(\delta_0 d^{-1/4}T^{3/4}\epsilon_0^{-1/2}\log^{1/4}(1/\delta_0)).
\end{align*}
\end{corollary}

\subsection{Proofs}

To prove Theorem~\ref{thm:amp}, we need the following important lemma, which can be seen as a special case of Theorem 3.8 in~\cite{feldman2020hiding}. In particular, in our paper,  we consider a fixed local randomizer rather than the more general adaptive one in~\cite{feldman2020hiding}. Another difference is that we consider the case of \emph{randomizer-then-shuffle} rather than the \emph{shuffle-then-randomizer}. However, as pointed in~\cite{feldman2020hiding}, the two cases are equivalent when the local randomizer is a fixed one. 

\begin{lemma}[Amplification by shuffling]
\label{lem:shuffle_amp}
Consider a one-round protocol $\cP=(\cR, \cS, \cA)$ over $n$ users. 
Let $\cR$ be an $(\epsilon_0,\delta_0)$-LDP mechanism. Then, for any $\delta' \in [0,1]$ such that $\epsilon_0 \le \log(\frac{n}{16\log(2/\delta')})$,  $\cP$ is $(\widetilde{\epsilon},\widetilde{\delta})$-SDP, i.e., the analyzer's view is $(\widetilde{\epsilon},\widetilde{\delta})$-DP, where 
\begin{align*}
    \widetilde{\epsilon} \le \log\left(1+ \frac{e^{\epsilon_0}-1}{e^{\epsilon_0}+1} \left(\frac{8\sqrt{e^{\epsilon_0}\log(4/\delta')}}{\sqrt{n}}  + \frac{8e^{\epsilon_0}}{n}\right)\right), \widetilde{\delta} = \delta' + (e^{\epsilon} + 1)\left(1+\frac{e^{-\epsilon_0}}{2}\right)n\delta_0,
\end{align*}
That is, when $\epsilon_0 > 1$, $\widetilde{\epsilon} = O\left(\frac{\sqrt{e^{\epsilon_0}\log(1/\delta') }}{\sqrt{n}}\right)$ and when $\epsilon_0 \le 1$, $\widetilde{\epsilon} = O\left(\epsilon_0\frac{\sqrt{\log(1/\delta')}}{\sqrt{n}}\right)$.
\end{lemma}

Roughly speaking, we have a privacy amplification by a factor $\sqrt{n}$ due to shuffling, which is the key to our analysis.

\begin{proof}[Proof of Theorem~\ref{thm:amp}]
To apply Lemma~\ref{lem:shuffle_amp}, we choose $\delta' = \delta/2$ and $\epsilon_0 = \frac{\epsilon \sqrt{B}}{\sqrt{\log(1/\delta')}} = \frac{\epsilon \sqrt{B}}{\sqrt{\log(2/\delta)}}$. For any $\epsilon \in (0,\sqrt{\log(2/\delta)}/\sqrt{B}]$, we have $\epsilon_0 \le 1$, which implies $\widetilde{\epsilon} = O(\epsilon)$. Meanwhile, we let $\delta_0 = \delta / B$ for any $\delta \in [0,1]$, which implies that $\widetilde{\delta} = O(\delta)$. Now, we are only left to choose $\sigma_1$ and $\sigma_2$ in $\cR_{\text{Amp}}$ so that it is $(\epsilon_0,\delta_0)$-LDP. To this end, via the standard Gaussian mechanism and boundedness assumption, we have when 
\begin{align*}
    \sigma_1 = \sigma_2 = \frac{4\sqrt{2\log(2.5/\delta_0)}}{\epsilon_0},
\end{align*}
$\cR_{\text{Amp}}$ is $(\epsilon_0,\delta_0)$-LDP. Finally, plugging in $\epsilon_0 = \epsilon \sqrt{B}/ \sqrt{\log(2/\delta)}$ and $\delta_0 = \delta/B$, yields 
\begin{align*}
    \sigma_1 = \sigma_2 = \frac{4\sqrt{2\log(2.5B/\delta)\log(2/\delta)} }{\epsilon\sqrt{B}}.
\end{align*}
Finally, plugging the value $\sigma = \frac{4\sqrt{2T\log(2.5B/\delta)\log(2/\delta)}}{\epsilon\sqrt{B}}$ (since there are total at most $T$ noise) into the regret bound in Lemma~\ref{cor:subG} yields the required results.
\end{proof}

\begin{proof}[Proof of Corollary~\ref{cor:amp_util}]
To establish the regret bound, we simply choose a balanced $B$ in the regret bound given by Theorem~\ref{thm:amp}. To prove the JDP guarantee, we will use the powerful \emph{Billboard lemma} (cf. Lemma 9 in~\cite{hsu2016private}), which says that an algorithm is JDP if the action recommended to each user is a function of her private data and a common signal computed in a differential private way. In our case, the private data is user's context and the common signal is the updated policy (i.e., $\hat{\theta}_m$ and design matrix $V_m$), which is a post-processing of shuffle outputs. Thus, the SDP guarantee directly implies the JDP guarantee in our case. Finally, the LDP guarantee simply follows from the standard Gaussian mechanism with parameter $\epsilon_0 = \epsilon \sqrt{B} /\sqrt{\log(2/\delta)}$ and $\delta_0 = \delta/B$.
\end{proof}

\begin{proof}[Proof of Corollary~\ref{cor:amp_priv}]
The LDP guarantee follows from standard Gaussian mechanism. To show the regret bound, we will use the result in Theorem~\ref{thm:amp}. In particular, comparing the values of $\sigma_1,\sigma_2$ in Corollary~\ref{cor:amp_priv} and the values in Theorem~\ref{thm:amp}, we can plug $\epsilon = \frac{\epsilon_0\sqrt{\log(2/\delta)}}{\sqrt{B}}$ and $\delta = \delta_0 B$ into the regret bound in Theorem~\ref{thm:amp}. Then, with a balanced choice of $B$, we obtain the required regret. The SDP guarantee also follows from Theorem~\ref{thm:amp} with $\epsilon = \frac{\epsilon_0\sqrt{\log(2/\delta)}}{\sqrt{B}}$ and $\delta = \delta_0 B$. Finally, as in the proof of Corollary~\ref{cor:amp_util}, the JDP guarantee follows from SDP guarantee and Billboard lemma.
\end{proof}

\section{Analysis of Vector Summation Protocol}\label{sec:app_vec}

\subsection{Pseudocode of $\cP_{\text{Vec}}$}
The shuffle protocol is given by $\mathcal{P}_{\text{Vec}} \!=\! (\cR_{\text{Vec}}, \cS_{\text{Vec}}, \cA_{\text{Vec}})$, in which $\cR_{\text{Vec}}$ is presented in Algorithm~\ref{alg:Rvec}, $\cS_{\text{Vec}}$ is presented in Algorithm~\ref{alg:Svec}, and $\cA_{\text{Vec}}$ is presented in Algorithm~\ref{alg:Avec}. Note that the original algorithm for the analyzer in~\cite{cheu2021shuffle} has a small issue in the de-bias process (cf. Algorithm 2 in~\cite{cheu2021shuffle}). In particular, instead of subtracting the norm $\Delta$, one needs to subtract $B \cdot \Delta$, see Lines 11 and 19 in Algorithm~\ref{alg:Avec}. Here, $B$ corresponds to $n$ in Algorithm 2 of~\cite{cheu2021shuffle}.
\begin{algorithm}[tb]
  \caption{Local Randomizer $\cR_{\text{Vec}}$ }
  \label{alg:Rvec}
\begin{algorithmic}[1]
  \STATE {\bfseries Parameters:} $g, b, p, d$
%   \STATE {\bfseries Initialize:} 
  \STATE \textcolor{gray}{\# Local randomizer for a scalar within $[0,\Delta]$}
  \FUNCTION{$\cR^*(x,\Delta)$}
    \STATE Set $\bar{x} = \lfloor{x g/\Delta}\rfloor$
    \STATE Sample rounding value $\gamma_1 \sim \textbf{Ber}(xg/\Delta -\bar{x})$
    \STATE Set $\hat{x} = \bar{x} + \gamma_1$
    \STATE Sample binomial noise $\gamma_2 \sim \textbf{Bin}(b,p)$
    \STATE Set $m$ be a multi-set containing $\hat{x} + \gamma_2$ copies of $1$ and $(g+b) - (\hat{x} + \gamma_2)$ copies of $0$. 
     \STATE {\bfseries return} $m$
  \ENDFUNCTION
   \FUNCTION{$R_1(\phi(c,a)y)$}
     \STATE Set $\Delta_1 = 1$
    \FOR{each coordinate $k \in [d]$}
        \STATE Shift data $w_k = [\phi(c,a)y]_k + \Delta_1$
        \STATE Run the scalar randomizer $m_k = \cR^*(w_k,\Delta_1)$
    \ENDFOR
    \STATE \textcolor{gray}{\# Labeled outputs (all bits in $m_k$ are labeled by $k$)}
    \STATE $M_1 = \{(k, m_k)\}_{k \in [d]}$ 
    \STATE {\bfseries return} $M_1$
  \ENDFUNCTION
  \FUNCTION{$R_2(\phi(c,a)\phi(c,a)^\top)$}
  \STATE Set $\Delta_2 = 1$
    \FOR{all $ i\le j \le d $}
        \STATE Shift data $w_{(i,j)} = [\phi(c,a)\phi(c,a)^\top]_{(i,j)} + \Delta_2$
        \STATE Run the scalar randomizer to obtain $m_{(i,j)} = \cR^*(w_{(i,j)}, \Delta_2)$ and $m_{(j,i)} = m_{(i,j)}$
    \ENDFOR
    \STATE \textcolor{gray}{\# Labeled outputs}
    \STATE $M_2 = \{((i,j), m_{(i,j)})\}_{(i,j) \in [d] \times [d]}$
    \STATE {\bfseries return} $M_2$
  \ENDFUNCTION
 
\end{algorithmic}
\end{algorithm}

\begin{algorithm}[!tb]
  \caption{Shuffler $\cS_{\text{Vec}}$ }
  \label{alg:Svec}
\begin{algorithmic}[1]
    \STATE {\bfseries Input:} $\{  M_{\tau, 1} \}_{\tau \in \cB}$ and $\{  M_{\tau, 2} \}_{\tau \in \cB}$, in which $\cB$ is a batch of users. $M_{\tau,1} = \{(k, m_k)\}_{k \in [d]}$ and $M_2 = \{((i,j), m_{(i,j)})\}_{(i,j) \in [d] \times [d]}$ are labeled data of user $\tau$
  \FUNCTION{$S_1( \{  M_{\tau, 1} \}_{\tau \in \cB})$}
   \STATE Uniformly permutes all messages, i.e., a total of $(g+b)\cdot B\cdot d$ bits
   \STATE Set $y_k$ be the collection of bits labeled by $k \in [d]$
   \STATE Set $Y_1 = \{y_1,\ldots, y_d\}$
    \STATE {\bfseries return} $Y_1$   
  \ENDFUNCTION
  \FUNCTION{$S_2( \{  M_{\tau, 2} \}_{\tau \in \cB})$}
    \STATE Uniformly permutes all messages, i.e., a total of $(g+b)\cdot B\cdot d^2$ bits
   \STATE Set $y_{(i,j)}$ be the collection of bits labeled by $(i,j) \in [d] \times [d]$
   \STATE Set $Y_2 = \{y_{(i,j)}\}_{(i,j) \in [d] \times d }$
    \STATE {\bfseries return} $Y_2$   
  \ENDFUNCTION
\end{algorithmic}
\end{algorithm}

\begin{algorithm}[tb]
  \caption{Analyzer $\cA_{\text{Vec}}$ }
  \label{alg:Avec}
\begin{algorithmic}[1]
    \STATE {\bfseries Input:} Shuffled outputs $ Y_1 = \{y_k\}_{k \in [d]}$ and $Y_2 = \{y_{(i,j)}\}_{(i,j) \in [d] \times d }$
    \STATE {\bfseries Initialize:} $g, b, p$
    \STATE \textcolor{gray}{\# Analyzer for a collection $y$ of $(g+b)\cdot B$ bits using $\Delta$}
  \FUNCTION{$\cA^*(y,\Delta)$}
     \STATE {\bfseries return} $\frac{\Delta}{g}\left( (\sum_{i=1}^{(g+b) \cdot B} y_i) - p\cdot b \cdot B\right)$
  \ENDFUNCTION
  \FUNCTION{$A_1( Y_1 )$}
    \STATE $\Delta_1 = 1$
    \FOR{each coordinate $k \in [d]$}
        \STATE Run analyzer on $k$-th labeled data to obtain $z_k = \cA^*(y_k,\Delta_1)$
        \STATE Re-center: $o_k = z_k - B \cdot \Delta_1$
    \ENDFOR
    \STATE {\bfseries return} $\{o_1,\ldots, o_k\}$
  \ENDFUNCTION
  \FUNCTION{$A_2( Y_2 )$}
    \STATE $\Delta_2 = 1$
    \FOR{all $ i\le j \le d $}
        \STATE Run analyzer on $(i,j)$-th labeled data to obtain $z_{(i,j)} = \cA^*(y_{(i,j)},\Delta_2)$
        \STATE Re-center: $o_{(i,j)} = z_{(i,j)} - B \cdot \Delta_2$ and $o_{(j,i)} = o_{(i,j)}$
    \ENDFOR
    \STATE {\bfseries return} $\{o_{(i,j)}\}_{(i,j) \in [d] \times [d]}$
  \ENDFUNCTION
\end{algorithmic}
\end{algorithm}

\subsection{Main Results}
\begin{theorem}[Restatement of Theorem~\ref{thm:vec-main}]
\label{thm:vec}
Fix batch size $B \!\in\! [T]$, privacy budgets $\epsilon \!\in\! (0,15]$, $\delta \!\in\! (0,1/2)$.
Then, Algorithm~\ref{alg:BOFUL} instantiated with $\cP_{\text{Vec}}$ with parameters $p \!=\! 1/4$, $g \!=\! \max\{2\sqrt{B}, d, 4\}$ and $b \!=\! \frac{24\cdot 10^4\cdot g^2\cdot \log^2\left(4 (d^2+1)/\delta\right)}{\epsilon^2B}$ is $(\epsilon,\delta)$-SDP. 
% there are choices of parameters $g,b \!\in\! \Nat$ and $p \!\in\! (0,1/2)$ depending on $B,\epsilon,\delta$ and feature dimension $d$ such that Algorithm~\ref{alg:BOFUL} instantiated with $\cP_{\text{Vec}}$ is $(\epsilon,\delta)$-SDP. 
 %  let $p = 1/4$,
% \begin{align*}
%     g = \max\{2\sqrt{B}, d, 4\},  b = \frac{24\cdot 10^4\cdot g^2\cdot \left(\log\left(\frac{4\cdot(d^2+1)}{\delta}\right)\right)^2}{\epsilon^2B}.
% \end{align*}
Furthermore, for any $\alpha \!\in\! (0,1]$, setting $\lambda \!=\! \Theta\!\left(\! \frac{\log(d^2/\delta)\sqrt{T}}{\epsilon\sqrt{B}} (\!\sqrt{d}\!+\! \sqrt{\log(T/(B\alpha))}\!\right)$, it enjoys the regret 
\begin{align*}
     \text{Reg}(T) \!=\! O\!\left(dB\log T \!+\! \frac{\log^{1/2}(d^2/\delta)}{\epsilon^{1/2}B^{1/4}} d^{3/4} T^{3/4}\log T\log(T/\alpha)\right)\!,
\end{align*}
with probability at least $1-\alpha$.
\end{theorem}

\begin{corollary}[Utility-targeted]
\label{cor:vec_util}
Under the same assumption in Theorem~\ref{thm:vec} and Algorithm~\ref{alg:BOFUL} is instantiated with $\cP_{\text{Vec}}$. Let $B = O( d^{-1/5}\epsilon^{-2/5}T^{3/5}\log^{2/5}(d^2/\delta))$, Algorithm~\ref{alg:BOFUL} achieves $(\epsilon,\delta)$-SDP with regret 
\begin{align*}
    Reg(T) = \widetilde{O}\left(d^{4/5} T^{3/5}\epsilon^{-2/5}\log^{2/5}(d^2/\delta)\right).
\end{align*}
Simultaneously, Algorithm~\ref{alg:BOFUL} also achieves $O(\epsilon,\delta)$-JDP and $O(\epsilon_0,\delta_0)$-LDP where 
\begin{align*}
    \epsilon_0 = O\left(\epsilon^{4/5} T^{3/10}d^{-1/10}\log^{1/5}(d^2/\delta) \right),\quad \delta_0 = O(\delta).
\end{align*}
\end{corollary}

\begin{corollary}[Privacy-targeted]
\label{cor:vec_priv}
Let Assumption~\ref{ass:bounded} hold and Algorithm~\ref{alg:BOFUL} is instantiated with $\cP_{\text{Vec}}$. For any $\epsilon_0 \in (0,15]$ and $\delta_0 \in (0,1/2)$, let 
\begin{align*}
    g = \max\{d, 4\}, \quad b = \frac{24\cdot 10^4\cdot g^2\cdot \left(\log\left(\frac{4\cdot(d^2+1)}{\delta_0}\right)\right)^2}{\epsilon_0^2},\quad p  = 1/4,
\end{align*}
Then, for all $B \in [T]$, Algorithm~\ref{alg:BOFUL} is $(\epsilon_0,\delta_0)$-LDP. Further suppose $B = O(d^{-1/4}T^{3/4}\epsilon_0^{-1/2}(\log(d^2/\delta_0))^{1/2})$, then Algorithm~\ref{alg:BOFUL} achieves regret 
\begin{align*}
     Reg(T) = \widetilde{O}\left( d^{3/4}  T^{3/4}\frac{\log^{1/2}(d^2/\delta_0)}{\sqrt{\epsilon_0}} \right).
\end{align*}
Simultaneously, Algorithm~\ref{alg:BOFUL} also achieves $O(\epsilon,\delta)$-SDP and $O(\epsilon,\delta)$-JDP where
\begin{align*}
    \epsilon = O\left(\epsilon^{5/4}T^{-3/8}d^{1/8}(\log(d^2/\delta_0))^{-1/4} \right), \quad \delta = O(\delta_0).
\end{align*}
\end{corollary}

\subsection{Proofs}
\begin{proof}[Proof of Theorem~\ref{thm:vec}]
The privacy part follows from the one-round SDP guarantee of vector summation protocol in~\cite{cheu2021shuffle}. In particular, by Theorem 3.2 in~\cite{cheu2021shuffle}, we have to properly choose parameters $g, b, p$ in $\cR_{\text{Vec}}$. To this end, by adapting the results of Lemma 3.1 in~\cite{cheu2021shuffle}, we have in our case when one chooses 
\begin{align*}
    g = \max\{2\sqrt{B}, d, 4\}, \quad b = \frac{24\cdot 10^4\cdot g^2\cdot \left(\log\left(\frac{4\cdot(d^2+1)}{\delta}\right)\right)^2}{\epsilon^2B},\quad p = 1/4,
\end{align*}
$\cP_{\text{Vec}}$ is $(\epsilon,\delta)$-SDP. It is worth pointing out that here we choose $b$ such that $p = 1/4$, which is necessary for our following analysis on the tail of the private noise. This is the key difference compared to the original one in~\cite{cheu2021shuffle} where the variance of the noise is sufficient. 

Now, we turn to regret analysis. Thanks to our general regret bound in Corollary~\ref{cor:subG}, we only need to verify the condition of sub-Gaussian private noise in the protocol $\cP_{\text{Vec}}$ (in particular $\cR_{\text{Vec}}$). To this end, we need a more careful analysis compared to~\cite{cheu2021shuffle} as the issue pointed above. Fix any coordinate $k \in [d]$, we will determine the private noise in $k$, which motivates us to check the scalar randomizer $\cR^*$ in $\cR_{\text{Vec}}$. Consider a batch of users. Let $z_i$ denote the sum of $g+b$ bits generated by user $i$ using $\cR^*$. That is, we have 
\begin{align*}
    z_i = \bar{x}_i + \gamma_{1,i} + \gamma_{2,i}.
\end{align*}
This implies that 
\begin{align*}
    z_i - bp = \frac{g}{\Delta}x_i +  \gamma_{1,i} + \bar{x}_i - \frac{g}{\Delta}x_i +\gamma_{2,i} - bp.
\end{align*}
Define shifted random variables $\iota_{1,i} := \gamma_{1,i} + \bar{x}_i - \frac{g}{\Delta}x_i $ and $\iota_{2,i}:= \gamma_{2,i} - bp$. Thus, taking the summation over all $i$ within a given batch $\cB$ of size $B$, yields
\begin{align*}
    \sum_{i\in \cB}z_i - B\cdot b\cdot p = \frac{g}{\Delta} \sum_{i \in \cB} x_i + \sum_{i\in \cB} \iota_{1,i} + \sum_{i\in \cB} \iota_{2,i},
\end{align*}
which implies that 
\begin{align*}
    \frac{\Delta}{g}\left(\sum_{i\in \cB} z_i - B\cdot b\cdot p \right) = \sum_{i \in \cB} x_i+ \frac{\Delta}{g} \sum_{i\in \cB} \iota_{1,i} + \frac{\Delta}{g}  \sum_{i\in \cB} \iota_{2,i}.
\end{align*}
Note that the above is exactly the output of the analyzer $\cA^*$ in $\cP_{\text{Vec}}$. Thus, to verify the sub-Gaussian condition in Assumption~\ref{ass:subG}, we only need to show that the last two terms above are zero-mean and sub-Gaussian random variables. To this end, we note that $\gamma_{1,i}$ is draw from \textbf{Ber}$(\frac{g}{\Delta} x_i - \bar{x}_i)$. Hence, $\ex{\iota_{1,i}} = 0$ and $\iota_{1,i}$ is sub-Gaussian with variance $1/4$ since $\iota_{1,i} \in [0,1]$. By independence of private noise across $i$, we have $\sum_{i \in \cB}\iota_{1,i}$ is sub-Gaussian with variance of $B/4$. Similarly, since $\gamma_{2,i}$ is independently sampled from binomial \textbf{Bin}$(b,p)$, we have $\ex{\iota_{2,i}} = 0$ and $\sum_{i\in \cB} \iota_{2,i}$ can be viewed as a sum of $B\cdot b$ bounded random variable within $[0,1]$, hence it is sub-Gaussian with variance of $B\cdot b /4$. Therefore, the total noise $ \frac{\Delta}{g} \sum_{i\in \cB} \iota_{1,i} + \frac{\Delta}{g}  \sum_{i\in \cB} \iota_{2,i}$
is sub-Gaussian with variance given by 
\begin{align*}
    \frac{\Delta^2}{g^2} \cdot \frac{B}{4} + \frac{\Delta^2}{g^2} \cdot B\cdot b/4 \ep{a}  \frac{1}{g^2} \cdot \frac{B}{4} + \frac{1}{g^2}\cdot B\cdot b/4 = O\left(\frac{(\log(d^2/\delta))^2}{\epsilon^2} \right).
\end{align*}
where (a) holds by the fact that in $\cP_{\text{Vec}}$, $\Delta = 1$. Thus, this implies that $\widetilde{\sigma}_1^2, \widetilde{\sigma}_2^2$ in Assumption~\ref{ass:subG} are satisfied with $O\left(M \frac{(\log(d^2/\delta))^2}{\epsilon^2} \right)$, hence $\sigma$ in Lemma~\ref{cor:subG} is given by $\sigma = O\left(\sqrt{T/B} \frac{\log(d^2/\delta)}{\epsilon} \right)$, which leads to the following regret bound 
\begin{align*}
    R(T) = \tilde{O}\left(dB + \frac{(\log(d^2/\delta))^{1/2}}{\sqrt{\epsilon}} d^{3/4} B^{-1/4} T^{3/4}\right),
\end{align*}
Hence, we finish the proof. 
\end{proof}

\begin{proof}[Proof of Corollary~\ref{cor:vec_util}]
The regret bound simply follows from a balanced choice of $B$ in Theorem~\ref{thm:vec}. As before, JDP follows from SDP and  Billboard lemma. To show the LDP guarantee, one way is to use DP property of Binomial mechanism and the refined advanced composition in~\cite{cheu2021shuffle} across dimensions (cf. Lemma 3.3 in~\cite{cheu2021shuffle}). However, there is a simple way to achieve this by noting that when $B=1$, the SDP guarantee of $\cP_{\text{Vec}}$ also implies LDP guarantee since now the shuffle output is the same as the output at each local randomizer\footnote{Here, we can assume that each local randomizer already randomly orders the $g+b$ bits before they are sent out.}. Thus, by comparing the values of $b$ for a general $B$ and the case when $B=1$, we can see that $\epsilon_0 = \epsilon\sqrt{B}$ and $\delta_0 = \delta$, i.e., an implicit privacy amplification by $\sqrt{B}$. Note that, this simple way might lead to a larger term in $\delta$. A careful analysis via Binomial mechanism and the (refined) advanced composition could yield something like $\epsilon_0 = \epsilon\sqrt{B}/\sqrt{\log(d^2/\delta)}$ and $\delta_0 = \delta/d^2$, where $d^2$ comes from the $d\times d$ matrix in the computation. Here we choose the simple way to avoid additional complexity for clarity.
\end{proof}

\begin{proof}[Proof of Corollary~\ref{cor:vec_priv}]
The LDP guarantee follows from the same trick as in the proof of Corollary~\ref{cor:vec_util} which helps to avoid Binomial mechanism and advance composition over dimensions. To establish the regret bound, we can compare the values of $b$ in Corollary~\ref{cor:vec_priv} and the one in Theorem~\ref{thm:vec}. In particular, we can plug $\epsilon = \frac{\epsilon_0}{\sqrt{B}}$ and $\delta = \delta_0$ into the regret bound in Theorem~\ref{thm:vec}. Then, with a balanced choice of $B$, we obtain the required regret. The SDP guarantee also follows from Theorem~\ref{thm:vec} with $\epsilon = \frac{\epsilon_0}{\sqrt{B}}$ and $\delta = \delta_0$. Finally, as in the proof of Corollary~\ref{cor:amp_util}, the JDP guarantee follows from SDP guarantee and Billboard lemma.
\end{proof}

\section{Joint Differenital Privacy}\label{app:JDP}
In this section, we will give formal DP definitions in the central model for linear contextual bandits. In particular, we first present the standard (event-level) definition which assumes all users are unique and then generalize it to (user-level) definition that allows for returning users. To this end, we first give the following general DP definition. 

\begin{definition}[General DP]
\label{def:generalDP}
A randomized mechanism $\cM : \cD \to \cR$ satisfies $(\epsilon,\delta)$-differential privacy if for any two adjacent datasets $X, X' \in \cD$ and for any measurable subsets of outputs $\cY \subseteq \cR$ it holds that 
\begin{align*}
    \prob{\cM(X) \in \cY} \le \exp(\epsilon) \prob{\cM(X') \in \cY} + \delta.
\end{align*}
\end{definition}

\begin{remark}
All the DP definitions in our main paper can be viewed as a particular instantiation of Definition~\ref{def:generalDP} in terms of  adjacent relation between two datasets and the corresponding output sequences. 
\end{remark}

% To adapt this classic central DP definition to privatize the learning protocol for linear contextual bandits, we need additional care since in contrast to one-round computation for a static dataset, the learning protocol in our case involves continual observations as well as multiple-round computation and communications. In particular, we need to carefully define the dataset and the output sequences. 
% under different trust models. Following previous works, we will assume all the $T$ users are unique in the main body of our paper and discuss how to handle returning users in Section [].

A straightforward adaptation of Definition~\ref{def:generalDP} to linear contextual bandits in the central model is to consider the sequence of $T$ unique users as the dataset, denoted by $U_T := \{u_1,\ldots, u_T\} \in \cU^T$, and the corresponding prescribed actions as the output sequence, denoted by $\cM(U_T) := \{a_1,\ldots, a_t\} \in \cA^T$. This is the central trust model because the learning agent in the protocol can have direct access to users' sensitive information, but all the prescribed actions via the deployed algorithm are indistinguishable on two neighboring user sequences. Unfortunately, it is not hard to see that this is in conflict with the goal of personalization of linear contextual bandits, which essentially requires the algorithm to prescribe different actions to different users according to their contexts. Indeed, as shown in~\cite{shariff2018differentially}, any learning protocol that satisfy the above notion of privacy protection has to incur a linear regret. Hence, to obtain a non-trivial utility-privacy trade-off, we need to relax DP to the notion called \emph{joint
differential privacy} (JDP)~\cite{kearns2014mechanism} in the central model, which requires that simultaneously for any user $u_t \in U_T$, the joint
distribution of the actions recommended to all users other than $u_t$ be differentially private in the type of the user $u_t$. It weakens the classic DP notion only in that the action suggested specifically to $u_t$ may be sensitive in her type (i.e., context and reward responses\footnote{Technically speaking, the type of the user is identified by the reward response she would give to all possible actions recommended based on her context information.}), as required by personalization. However, JDP is still a very strong definition since it protects $u_t$ from any arbitrary collusion of other users against her, so long as she does not herself reveal the action suggested to her. Formally, we let $\cM_{-t}(U_T) := \cM(U_T) \setminus \{a_t\} $ to denote all the actions prescribed by the deployed algorithm excluding the one  recommended to $u_t$ and based on it we have the definition of JDP as follows. 
\begin{definition}[Joint Differential Privacy (JDP)]
\label{def:JDP}
A learning process of linear contextual bandits is $(\epsilon,\delta)$-joint differentially private if its deployed algorithm $\cM:\cU^T \to \cA^T$ satisfies that for all $t \in [T]$, for all neighboring user sequences $U_T, U_T' \in \cU^T$ differing only on the $t$-th user and for all set of actions $\cA_{-t} \subseteq \cA^{T-1}$ given to all but the $t$-th user, \begin{align*}
    \prob{\cM_{-t}(U_T) \in \cA_{-t}} \le \exp(\epsilon) \prob{\cM_{-t}(U_T') \in \cA_{-t}} + \delta.
\end{align*}
\end{definition}

The above JDP definition assumes that all the $T$ users are unique, which is the standard event-level DP considered in existing similar works~\cite{shariff2018differentially,vietri2020private,sayakPO}. That is, since each user only contributes one event in the total $T$ rounds, two user sequences $U_T$ and $U'_T$ are said to be adjacent if they only differ at one round $t \in [T]$.

However, a more practical situation is that one user could contribute her data at multiple rounds, i.e., returning users. This motivates us to consider a user-level JDP, in which two user sequences $U_T$ and $U'_T$ are adjacent if one replaces all the data associated with user $u$ to $u'$ in $U_T$ results in $U'_T$. In this case, changing one user in the sequence could affect the data at multiple rounds. Accordingly, the output sequences need to remove all the actions at these rounds to avoid the conflict with personalization. Following the notations in~\cite{dwork2010differential}, we say $U_T$ and $U'_T$ are neighboring sequences if there exist $u, u'$ such that if one replace some of $u$ in $U_T$, the resultant sequence is $U'_T$. Formally, $U_T , U'_T$ are neighboring with neighboring indices $\cI$,  if there exist $u,u' \in \cU$ and index set $\cI \subseteq [T]$ such that $U_T|_{\cI: u \to u'} = U'_T$, in which  $U_T|_{\cI: u \to u'}$ means replacing $u$ by $u'$ in $U_T$ at all  indices in $\cI$. Meanwhile, we let $\cM_{-\cI}(U_T) := \cM(U_T) \setminus a_{\cI} $, where $a_{\cI}$ is the set of actions at indices in $\cI$. With these notations, we have the following formal definition. 

% Formally, we let $\cM_{-t}(U_T) := \cM(U_T) \setminus \{a_t\} $ to denote all the actions prescribed by the deployed algorithm excluding the one  recommended to $u_t$ and based on it we have the definition of JDP as follows. 
\begin{definition}[User-level JDP]
\label{def:JDP-user}
A learning process of linear contextual bandits is $(\epsilon,\delta)$-joint differentially private if its deployed algorithm $\cM:\cU^T \to \cA^T$ satisfies that for  all neighboring user sequences $U_T, U_T' \in \cU^T$ with neighboring indices given by $\cI$, and for all set of actions $\cA_{-\cI} \subseteq \cA^{T-|\cI|}$,
\begin{align*}
    \prob{\cM_{-\cI}(U_T) \in \cA_{-\cI}} \le \exp(\epsilon) \prob{\cM_{-\cI}(U_T') \in \cA_{-\cI}} + \delta.
\end{align*}
\end{definition}

\begin{remark}
A straightforward way to achieve user-level JDP via event-level JDP is to use group privacy property of DP~\cite{dwork2014algorithmic,vietri2020private}. In particular, suppose a mechanism is $(\epsilon,\delta)$-JDP (event-level), then it is $(k\epsilon,ke^{(k-1)\epsilon}\delta)$-JDP (user-level) if each user contributes at most $k$ rounds. This black-box approach leads to a large increase in $\delta$. We will show that a careful and direct analysis can improve this part while the linear increase in $\epsilon$ is unchanged. This makes sense since the sensitivity now increases by a factor of $k$.
\end{remark}

\section{Regret and Privacy Analysis Under Returning Users}\label{app:return}

We consider the following returning users case.

\begin{assumption}[Returning Users]
\label{ass:return}
Fix a batch size $B$, any particular user can potentially participates in \emph{all} $M = T/B$ batches, but within each batch $m \in [M]$, she only contributes once.
\end{assumption}

% \begin{theorem}[SDP under returning users]
% \label{thm:return-main}
% Let Assumption~\ref{ass:bounded} and Assumption~\ref{ass:return_ass} hold. Then, we have the following results for $\cP_{\text{Amp}}$ and $\cP_{\text{Vec}}$, respectively. 
% \begin{enumerate}
%     \item For any $\epsilon \in [0,\frac{2}{B}\sqrt{2T\log(2/\delta)}]$, $\delta \in (0,1]$ and $B = O(d^{-1/6}\epsilon^{-1/3}T^{2/3}(\log(T/\delta))^{1/3})$, let $\sigma_1 = \sigma_2 = \frac{16\sqrt{T\log(2/\delta)(\log(5T/\delta))}}{\epsilon{B}}$. Then, $\cP_{\text{Amp}}$ is $O(\epsilon,\delta)$-SDP with regret 
% \begin{align*}
%      R(T) = \tilde{O}\left(d^{5/6} T^{2/3}\epsilon^{-1/3}\left(\log(T/\delta)\right)^{1/3}\right).
% \end{align*}
%     \item For any $\epsilon \le 15$, $\delta \in (0,1/2)$ and $B = O(d^{-1/6}\epsilon^{-1/3}T^{2/3}(\log(Td^2/\delta))^{1/2})$, let 
% \begin{align*}
%     g = \max\{2\sqrt{B}, d, 4\}, \quad b = \frac{ 10^7\cdot \log(2/\delta)\cdot g^2 \cdot T\cdot \left(\log\left(\frac{8\cdot T(d^2+1)}{B\delta}\right)\right)^2}{\epsilon^2B^2},\quad p  = 1/4.
% \end{align*}
% Then, $\cP_{\text{Vec}}$ is $(\epsilon,\delta)$-SDP. Meanwhile, the corresponding regret satisfies 
% \begin{align*}
%       R(T) = \tilde{O}\left(d^{5/6} T^{2/3}\epsilon^{-1/3}\left(\log(d^2T/\delta)\right)^{1/2}\right).
% \end{align*}
% \end{enumerate}
% \end{theorem}

Under the above assumption, our previous SDP guarantee from one-round SDP protocol is no longer true. Instead, we now need to guarantee that outputs of all the batches together have a total privacy loss of $(\epsilon,\delta)$, since all of them may reveal the sensitive information of a given user if she participates in all the batches, i.e., worst-case scenario. To this end, we resort to advanced composition theorem~\cite{dwork2014algorithmic}, which is restated as follows for an easy reference.

\begin{theorem}[Advanced composition]
\label{thm:composition}
Given target privacy parameters $\epsilon' \in (0,1)$ and $\delta'>0$, to ensure $(\epsilon', k\delta + \delta')$-DP for the
composition of $k$ (adaptive) mechanisms, it suffices that each mechanism is $(\epsilon,\delta)$-DP with $\epsilon = \frac{\epsilon'}{2\sqrt{2k\log(1/\delta')}}$.
\end{theorem}

\subsection{LDP Amplification Protocol}
\begin{theorem}[Formal statement of (i) in Theorem~\ref{thm:return-main}]
\label{thm:amp_return}
Let Assumption~\ref{ass:bounded} and Assumption~\ref{ass:return} hold. For any $\epsilon \in [0,\frac{2}{B}\log(2/\delta)\sqrt{2T}]$, $\delta \in (0,1]$ and $B \in [T]$, let $\sigma_1 = \sigma_2 = \frac{16\log(2/\delta)\sqrt{T(\log(5T/\delta))}}{\epsilon{B}}$. Then, Algorithm~\ref{alg:BOFUL} instantiated using $\cP_{\text{Amp}}$ is $O(\epsilon,\delta)$-SDP. Furthermore, for any $\alpha \in (0,1]$, setting $\lambda =  \Theta(\sqrt{T}\sigma_1 (\sqrt{d}+ \sqrt{\log(T/B\alpha)})$, it has the following regret
\begin{align*}
    \text{Reg}(T) = {O}\left(\frac{dT}{M}\log T + \sqrt{\frac{MT}{\epsilon}}d^{3/4} \log^{1/4} (T/\delta) \log^{1/2}(2/\delta)\log T \log(T/\alpha)\right).
\end{align*}
% \begin{align*}
%      R(T) = \tilde{O}\left(dB + \epsilon^{-1/2}Td^{3/4}B^{-1/2}(\log(T/\delta))^{1/2}\right)
% \end{align*}
\end{theorem}
% \begin{corollary}
% Under the same assumption in Theorem~\ref{thm:amp}. Let $B=1$, $\cP_{\text{Amp}}$ reduces to a standard $O(\epsilon,\delta)$-LDP protocol in the sequential setting with regret $\tilde{O}(\frac{\log(T/\delta)^{1/4}}{\sqrt{\epsilon}} d^{3/4}  T^{3/4})$
% \end{corollary}
The following corollary says that if the batch schedule also depends on privacy parameters, one can improve the dependence on $\epsilon$, i.e., from $\epsilon^{-1/2}$ to $\epsilon^{-1/3}$. 
\begin{corollary}[Utility-targeted]
\label{cor:amp_util_return}
Under the same assumption in Theorem~\ref{thm:amp_return} and $B = O(d^{-1/6}\epsilon^{-1/3}T^{2/3}(\log(T/\delta))^{1/2})$, Algorithm~\ref{alg:BOFUL} instantiated using $\cP_{\text{Amp}}$ achieves $O(\epsilon,\delta)$-SDP with regret 
\begin{align*}
    R(T) = \tilde{O}\left(d^{5/6} T^{2/3}\epsilon^{-1/3}\left(\log(T/\delta)\right)^{1/2}\right).
\end{align*}
% Simultaneously, $\cP_{\text{Amp}}$ also achieves $O(\epsilon,\delta)$-JDP and $O(\epsilon_0,\delta_0)$-LDP where 
% \begin{align*}
%     \epsilon_0= O\left(\epsilon^{2/3}T^{1/6} d^{-1/6} (\log(T/\delta))^{1/3} \right), \quad \delta_0 = O\left(\delta T^{-1}\right).
% \end{align*}
\end{corollary}

% \begin{corollary}[Privacy-targeted]
% \label{cor:amp_priv_return}
% Let Assumption~\ref{ass:bounded} hold. For any $\epsilon_0 \in [0,1]$ and $\delta_0 \in (0,1]$, let $\sigma_1 = \sigma_2 = \frac{4\sqrt{2\log(2.5/\delta_0)}}{\epsilon_0}$. Then, for all $B \in [T]$, $\cP_{\text{Amp}}$ is $(\epsilon_0,\delta_0)$-LDP. Further suppose $B = O(d^{-1/4}T^{3/4}\epsilon_0^{-1/2}(\log(1/\delta_0))^{1/4})$, then 
%  $\cP_{\text{Amp}}$ achieves regret  
% \begin{align*}
%     R(T) = \tilde{O}\left(d^{3/4}T^{3/4}\epsilon_0^{-1/2}(\log(1/\delta_0))^{1/4}\right).
% \end{align*}
% Simultaneously, $\cP_{\text{Amp}}$ achieves $O(\epsilon,\delta)$-SDP and $O(\epsilon,\delta)$-JDP where 
% \begin{align*}
%     \epsilon = O\left(\epsilon_0^{3/2}T^{-1/4} d^{1/4} (\log(1/\delta_0))^{1/4}\right), \quad \delta = O(T\delta_0).
% \end{align*}
% \end{corollary}

\begin{proof}[Proof of Theorem~\ref{thm:amp_return}]
First, by advanced composition in Theorem~\ref{thm:composition}, if we let each batch's privacy parameters be $\epsilon_m = \frac{\epsilon}{2\sqrt{2M\log(2/\delta)}}$ and $\delta_m = \delta/(2M)$, then final privacy guarantee is $(\epsilon,\delta)$-DP. Thus, we only need to replace $\epsilon$ by $\epsilon_m$ and $\delta$ by $\delta_m$ in Theorem~\ref{thm:amp}
\end{proof}

% \begin{proof}[Proof of Corollary~\ref{cor:amp_util_return}]
% The regret bound follows directly from Theorem~\ref{thm:amp_return} with a balanced choice of batch size $B$. As usual, the JDP guarantee follows from SDP guarantee and the application of Billboard lemma. For LDP, we note that by the change of variables, we can rewrite $\sigma_1 = \sigma_2 = \frac{4\sqrt{2\log(2.5/\delta_0)}}{\epsilon_0}$ where $\epsilon_0 = B\epsilon/(2\sqrt{2T\log(2/\delta)})$ and $\delta_0 = \delta/(2T)$. Thus, by standard Gaussian mechanism, we have $\cP_{\text{Vec}}$ achieves $(\epsilon_0,\delta_0)$-LDP. 
% \end{proof}

% \begin{proof}[Proof of Corollary~\ref{cor:amp_priv_return}]
% The LDP guarantee follows from standard Gaussian mechanism. To show the regret bound, we will use the result in Theorem~\ref{thm:amp_return}. In particular, comparing the values of $\sigma_1,\sigma_2$ in Corollary~\ref{cor:amp_priv_return} and the values in Theorem~\ref{thm:amp_return}, we can plug $\epsilon = \frac{2\epsilon_0\sqrt{2T\log(2/\delta)}}{{B}}$ and $\delta = 2T\delta_0$ into the regret bound in Theorem~\ref{thm:amp}. Then, with a balanced choice of $B$, we obtain the required regret. The SDP guarantee also follows from Theorem~\ref{thm:amp} with $\epsilon = \frac{2\epsilon_0\sqrt{2T\log(2/\delta)}}{{B}}$ and $\delta = 2T\delta_0$. Finally, as in the proof of Corollary~\ref{cor:amp_util}, the JDP guarantee follows from SDP guarantee and Billboard lemma.
% \end{proof}

\subsection{Vector Summation Protocol}
\begin{theorem}[Formal statement of (ii) in Theorem~\ref{thm:return-main}]
\label{thm:vec_return}
Let Assumption~\ref{ass:bounded} and Assumption~\ref{ass:return} hold. Then, for any $\epsilon \le 15$, $\delta \in (0,1/2)$ and $B \in [T]$, let 
\begin{align*}
    g = \max\{2\sqrt{B}, d, 4\}, \quad b = \frac{ 10^7\cdot \log(2/\delta)\cdot g^2 \cdot T\cdot \left(\log\left(\frac{8\cdot T(d^2+1)}{B\delta}\right)\right)^2}{\epsilon^2B^2},\quad p  = 1/4.
\end{align*}
Algorithm~\ref{alg:BOFUL} instantiated using $\cP_{\text{Vec}}$ is $(\epsilon,\delta)$-SDP. Furthermore, for any $\alpha \in (0,1]$, setting 
\begin{align*}
    \lambda = \Theta\left(\frac{T\sqrt{\log(2/\delta)}\log(d^2T/(B\delta))}{B} \left(\sqrt{d} + \sqrt{\log(T/(B\alpha))}\right)\right),
\end{align*}
% $\lambda = \Theta(\frac{T\sqrt{\log(2/\delta)}\log(d^2T/(B\delta))}{B})$,  
then it has the regret bound
\begin{align*}
     \text{Reg}(T) \!=\! {O}\left(\frac{dT}{M}\log T \!+\!\sqrt{\frac{MT}{\epsilon}}d^{3/4} \log^{3/4} (d^2M/\delta)\log T \log(T/\alpha) \right).
\end{align*}
% \begin{align*}
%      R(T) = \tilde{O}\left(dB + \epsilon^{-1/2}B^{-1/2}Td^{3/4}(\log(d^2T/(B\delta)))^{3/4}\right).
% \end{align*}
\end{theorem}

\begin{corollary}[Utility-targeted]
Under the same assumption in Theorem~\ref{thm:vec_return}, $B = O(d^{-1/6}\epsilon^{-1/3}T^{2/3}(\log(Td^2/\delta))^{1/2})$, Algorithm~\ref{alg:BOFUL} instantiated using $\cP_{\text{Vec}}$ achieves $(\epsilon,\delta)$-SDP with regret 
\begin{align*}
    R(T) = \widetilde{O}\left(d^{5/6} T^{2/3}\epsilon^{-1/3}\left(\log(d^2T/\delta)\right)^{1/2}\right).
\end{align*}
% Simultaneously, $\cP_{\text{Vec}}$ also achieves $O(\epsilon,\delta)$-JDP and $O(\epsilon_0,\delta_0)$-LDP where 
% \begin{align*}
%     \epsilon_0 = O\left(\epsilon^{2/3}T^{1/6} d^{-1/6}(\log(d^2T/\delta))^{1/2} \right),\quad \delta_0 = O\left(\delta T^{-1/3}\epsilon^{-1/3}d^{-1/6} (\log(d^2T/\delta))^{1/2} \right).
% \end{align*}
\end{corollary}

% \begin{corollary}[Privacy-targeted]
% Let Assumption~\ref{ass:bounded} hold. For any $\epsilon_0 \in [0,15]$ and $\delta_0 \in (0,1/2)$, let 
% \begin{align*}
%     g = \max\{d, 4\}, \quad b = \frac{24\cdot 10^4\cdot g^2\cdot \left(\log\left(\frac{4\cdot(d^2+1)}{\delta_0}\right)\right)^2}{\epsilon_0^2},\quad p  = 1/4,
% \end{align*}
%  Then, for any $B\in [T]$, $\cP_{\text{Vec}}$ is $(\epsilon_0,\delta_0)$-LDP. Further suppose $B = O(d^{-1/4}T^{3/4}\epsilon_0^{-1/2}(\log(d^2/\delta_0))^{1/2})$, then $\cP_{\text{Vec}}$ achieves regret 
% \begin{align*}
%      R(T) = \tilde{O}\left( d^{3/4}  T^{3/4}\frac{(\log(d^2/\delta_0))^{1/2}}{\sqrt{\epsilon_0}} \right).
% \end{align*}
% Simultaneously, $\cP_{\text{Vec}}$ also achieves $O(\epsilon,\delta)$-SDP and $O(\epsilon,\delta)$-JDP where
% \begin{align*}
%     \epsilon = O\left(\epsilon^{3/2}T^{-1/4}d^{1/4} \right), \quad \delta = O(\delta_0 T^{3/4}d^{1/4}\epsilon^{1/2}(\log(d^2/\delta_0))^{-1/2}).
% \end{align*}
% \end{corollary}

\begin{proof}[Proof of Theorem~\ref{thm:vec_return}]
First, by advanced composition in Theorem~\ref{thm:composition}, if we let each batch's privacy parameters be $\epsilon_m = \frac{\epsilon}{2\sqrt{2M\log(2/\delta)}}$ and $\delta_m = \delta/(2M)$, then final privacy guarantee is $(\epsilon,\delta)$-DP. Thus, we only need to replace $\epsilon$ by $\epsilon_m$ and $\delta$ by $\delta_m$ in Theorem~\ref{thm:vec}
\end{proof}

\subsection{JDP under Returning Users}
% \xingyu{I double-check that the refined analysis is not applicable to the group privacy, which is different from composition. Very subtle thing. Thus, the only thing is to use group privacy. See the remark after Theorem 2.2 of Dwork's book}

As mentioned before, existing algorithm with JDP guarantee assumes \emph{unique} users, i.e., event-level JDP given by Definition~\ref{def:JDP}. To handle returning users, we need to consider user-level JDP given by Definition~\ref{def:JDP-user}. One straightforward way is to resort to group privacy~\cite{dwork2014algorithmic}. That is, if any user appears at most $M_0$ rounds in the process, the original $(\epsilon, \delta)$-JDP algorithm proposed in~\cite{shariff2018differentially} now achieves $(M_0\epsilon, M_0\exp((M_0-1)\epsilon)\delta)$-JDP (user-level). However, this black-box will incur a large loss in the $\delta$ term. To overcome this, we note that a simple modification of the added noise in the original algorithm in~\cite{shariff2018differentially} will work. In particular, we scale up the noise variance by a multiplicative factor of $M_0^2$, if any user participates in at most $M_0$ rounds. This follows from the fact that flipping one user now would change the $\ell_2$ sensitivity of the expanded binary-tree nodes from $O(\sqrt{\log T})$ to $O(M_0\sqrt{\log T})$. Then, utilizing our derived generic regret bound in Lemma~\ref{cor:subG}, yields the following result.
% Roughly speaking, we have to decrease the privacy budget by a factor of $1/M$ to attain the same privacy guarantee as before.

\begin{proposition}[Restatement of Proposition \ref{prop:JDP-return}]
If any user participates in at most $M_0$ rounds, the algorithm in~\cite{shariff2018differentially}, with the above modification to handle user-level privacy, achieves the high-probability regret bound
\begin{align*}
   \text{Reg}(T) = \widetilde{O}\left(d\sqrt{T} + \sqrt{\frac{M_0T}{\epsilon}}d^{3/4} \log^{1/4} (1/\delta) \right).
\end{align*}
\end{proposition}

\begin{proof}
The key idea behind the regret analysis in the central model for linear contextual bandits in~\cite{shariff2018differentially} is to utilize the following two properties of the so-called tree-based mechanism (or binary counting mechanism)~\cite{chan2010private}: (i) change of each leaf-node (corresponding to a user's data) only incurs the change of $l_2$-sensitivity of the expanded binary-tree by $O(\sqrt{\log T})$; (ii) for any $t \in [T]$, the summation of data from time $1$ to $t$ only involves at most $O({\log T)}$ tree nodes. Property (i) is used to compute the added noise at each node to guarantee privacy while property (ii) is used to compute the total noise in the private sum when bounding the regret. Now, in the case of returning users, if we flip one user's data, it will change the $l_2$-sensitivity of the expanded binary-tree by $O(M_0\sqrt{\log T})$, i.e., an additional $M_0$ factor in the sensitivity, which leads to the additional $M_0^2$ factor in the added noise. Property (ii) is the same as before, i.e., total number of noise is at most $O(\log T)$. Finally, by Lemma~\ref{cor:subG}, we have the result. 
% In particular, when calculating the added noise, previous works~\cite{shariff2018differentially} mainly use \emph{concentrated differential privacy} to remove the additional $\log(1/\delta)$ factor in the standard advanced composition for Gaussian mechanism. 
\end{proof}

\section{Batched Algorithms for Local and Central Models}
\label{sec:batched}
To start with, for the batched algorithm in the local model, one can simply replace the shuffler in Algorithm~\ref{alg:BOFUL} by an identity mapping while using the same local randomizer as in~\cite{zheng2020locally} (i.e., Gaussian mechanism). We call this algorithm $\emph{Batched-Local-LinUCB}$. Thanks to Lemma~\ref{cor:subG}, we have the following privacy and regret guarantees. 
\begin{proposition}
Let Assumption~\ref{ass:bounded} hold. Fix any $\epsilon_0 \in [0,1]$, $\delta_0 \in (0,1]$ and $\alpha \in [0,1]$, let $\sigma_1 = \sigma_2 = \frac{4\sqrt{2\log(2.5/\delta_0)}}{\epsilon_0}$. Then, for all $B \in [T]$, Bathed-Local-LinUCB is $(\epsilon_0,\delta_0)$-LDP and with probability at least $1-\alpha$
\begin{align*}
    \reg(T) = \tilde{O}\left(dB + T^{3/4}d^{3/4}\frac{(\log(1/\delta))^{1/4}}{\sqrt{\epsilon}}\log(T/\alpha)\right).
\end{align*}
\end{proposition}

\begin{remark}
The above theorem indicates that it suffices to update every $B = \tilde{O}(T^{3/4})$ to ensure the same privacy and regret guarantees as in the sequential case.
\end{remark}

For the batched algorithm in the central model, we can make the following simple modification over the sequential one in~\cite{shariff2018differentially}, which relies on the seminal tree-based algorithm~\cite{chan2010private} at the central server (analyzer) to balance between privacy and regret. In the batched case, instead of updating the binary-tree nodes after every round, the server updates them only after each batch by treating the the sum of the statistics (i.e., vectors or matrices) within the batch as a single new observation. We call this algorithm \emph{Batched-Central-LinUCB}.
With this modification and Lemma~\ref{cor:subG}, we have the following privacy and regret guarantees. 
% That is, the binary-tree now only has $M$ leaf nodes and each internal node 
\begin{proposition}
Let Assumption~\ref{ass:bounded} hold. Fix any $\epsilon \in [0,1]$, $\delta \in (0,1]$ and $\alpha \in [0,1]$. Then, for all $B \in [T]$, Bathed-Central-LinUCB is $(\epsilon,\delta)$-JDP and with probability at least $1-\alpha$
\begin{align*}
    \reg(T) = \tilde{O}\left(dB + \sqrt{T}d^{3/4}\frac{(\log(1/\delta))^{1/4}}{\sqrt{\epsilon}}\log(T/\alpha)\right).
\end{align*}
\end{proposition}
\begin{remark}
The above theorem indicates that it suffices to update every $B = \tilde{O}(\sqrt{T})$ to attain the same privacy-regret trade-off as in the sequential case.
\end{remark}

\section{Additional Experimental Results}\label{app:sim}

\begin{figure}[h]
% 		\begin{subfigure}[t]{.3\linewidth}
% 			\includegraphics[width = 2.0in]{Figures/Bernoulli (d=5,eps=0.2,delta=0.1,B=20).pdf}
% 			\caption{1 }
% 		\end{subfigure}\ \
% 		\begin{subfigure}[t]{.3\linewidth}
% 			\includegraphics[width = 2.0in]{Figures/Bernoulli (d=5,eps=1,delta=0.1,B=20).pdf}
% 			\caption{2}
% 		\end{subfigure}\ \
% 		\begin{subfigure}[t]{.3\linewidth}
% 			\includegraphics[width = 2.0in]{Figures/Bernoulli (d=5,eps=10,delta=0.1,B=20).pdf}
% 			\caption{3}
% 		\end{subfigure}\\
% 		 \vspace{0mm}
% 		\vspace{0mm}
		\begin{subfigure}[b]{.48\textwidth}
		\centering
			\includegraphics[width = 2.1in]{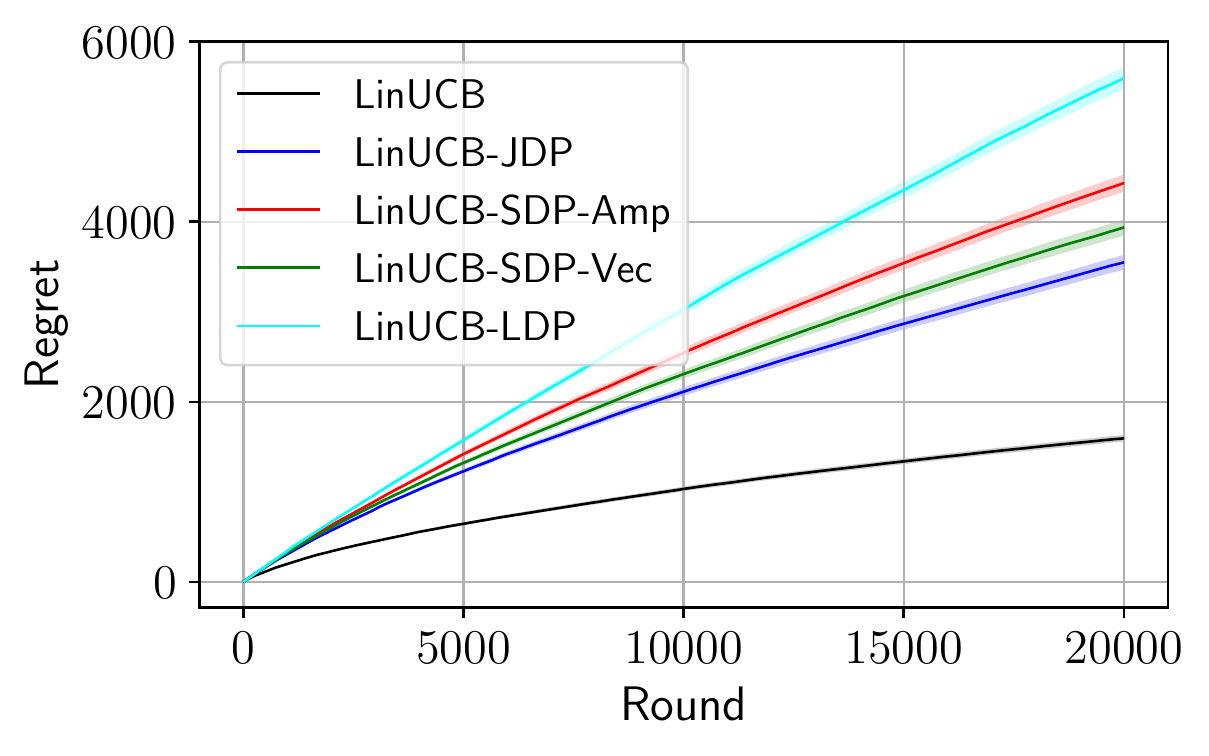}
			\caption{$d=10$}
		\end{subfigure} 
		\begin{subfigure}[b]{.48\textwidth}
		\centering
			\includegraphics[width = 2.1in]{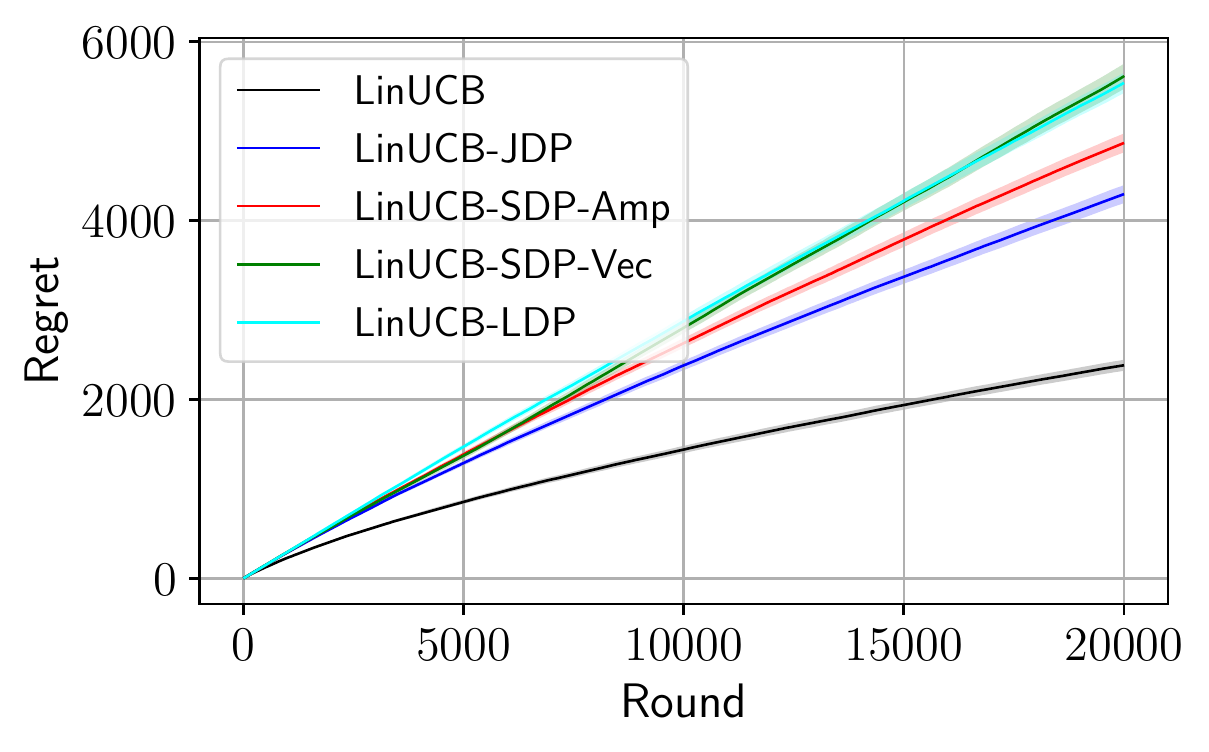}
			\caption{$d=15$}
		\end{subfigure} 
% 		\begin{subfigure}[t]{.3\linewidth}
% 			\includegraphics[width = 2.1in]{Figures/new_eps_10.pdf}
% 			\caption{$\epsilon = 10$}
% 		\end{subfigure}
		 \vspace{0mm}
		\caption{\footnotesize{Comparison of cumulative regret for LinUCB (non-private), LinUCB-JDP (central model), LinUCB-SDP (shuffle model) and LinUCB-LDP (local model) with privacy level $\epsilon=1$ for varying feature dimension $d=10$ (a) and $d=15$ (b). In all cases, regret of LinUCB-SDP lies perfectly in between LinUCB-JDP and LinUCB-LDP, achieving finer regret-privacy trade-off.}  }\label{fig:all_algos_extra}
		\vspace{0mm}
\end{figure}
\end{appendix}

\end{document}